\documentclass[a4paper,11pt]{scrartcl}

\RequirePackage{xr}
\RequirePackage[OT1]{fontenc}
\RequirePackage{amsthm,amsmath,amssymb,dsfont,bm}
\usepackage{mathbbol}
\DeclareSymbolFontAlphabet{\amsmathbb}{AMSb}%
\RequirePackage{algorithm}
\RequirePackage{algpseudocode}
\RequirePackage{graphicx} 
\RequirePackage[section]{placeins}
\RequirePackage{color}
\RequirePackage{chngcntr}
\RequirePackage{enumerate}
\usepackage[shortlabels]{enumitem}
\RequirePackage{longtable}
\usepackage{multirow}
\usepackage{mathtools}
\usepackage[english]{babel} 

\RequirePackage[OT1]{fontenc} 

\RequirePackage{bbm}
\usepackage[normalem]{ulem}

\RequirePackage{graphicx}
\RequirePackage{verbatim}
\usepackage{url}

\numberwithin{equation}{section}
\RequirePackage{a4wide} 
\usepackage{color,charter,graphicx,fancyhdr,bibentry,enumitem,eso-pic,here,setspace,multirow,amsmath,amssymb,amsfonts,subfigure,mathdots,sidecap}
\usepackage[svgnames,x11names]{xcolor}
\usepackage[pdftex]{hyperref}

\theoremstyle{plain}
\newtheorem{definition}{Definition}[section]
\newtheorem{theorem}{Theorem}[section]

\newtheorem{corollary}{Corollary}[section]

\newtheorem{remark}{Remark}[section]
\newtheorem{lemma}{Lemma}[section]

\def\@bysame#1{\vrule height 1.5pt depth -1pt width 3em \hskip
0.5em\relax}

\newcommand{\N}{ \mathbb{N} }

\newcommand{\Z}{ \mathbb{Z} }

\newcommand{\R}{ \mathbb{R} }

\newcommand{\pen}{\operatorname{pen}}

\newcommand{\wh}[1]{ \widehat{ #1 } }
\newcommand{\wt}[1]{ \widetilde{ #1 } }

\newcommand{\calA}{\mathcal{A}}

\newcommand{\calF}{\mathcal{F}}

\newcommand{\calI}{\mathcal{I}}

\newcommand{\calN}{\mathcal{N}}

\newcommand{\eins}{{\bm 1}}

\newcommand{\matA}{{\mathbb{A}}}
\newcommand{\matB}{{\mathbb{B}}}
\newcommand{\matC}{{\mathbb{C}}}
\newcommand{\matD}{{\mathbb{D}}}

\newcommand{\matF}{{\mathbb{F}}}

\newcommand{\matO}{{\mathbb{O}}}

\newcommand{\matR}{{\mathbb{R}}}

\newcommand{\matU}{{\mathbb{U}}}
\newcommand{\matV}{{\mathbb{V}}}
\newcommand{\matW}{{\mathbb{W}}}
\newcommand{\matX}{\mathbb{X} }
\newcommand{\matY}{{\mathbb{Y}}}
\newcommand{\matZ}{{\mathbb{Z}}}
\newcommand{\matPhi}{{\mathbb{\Phi}}}

\newcommand{\matid}{\mathbb{I}}

\newcommand{\vecnull}{{\bm 0}}

\newcommand{\veca}{{\bm a}}
\newcommand{\vecb}{{\bm b}}

\newcommand{\vecc}{{\bm c}}

\newcommand{\vece}{{\bm e}}
\newcommand{\vecE}{{\bm E}}
\newcommand{\vecF}{{\bm F}}
\newcommand{\vecf}{{\bm f}}
\newcommand{\vecg}{{\bm g}}

\newcommand{\vecH}{{\bm H}}

\newcommand{\veco}{{\bm o}}

\newcommand{\vecu}{{\bm u}}
\newcommand{\vecr}{{\bm r}}

\newcommand{\vecs}{{\bm s}}

\newcommand{\vecv}{{\bm v}}

\newcommand{\vecw}{{\bm w}}
\newcommand{\vecx}{{\bm x}}
\newcommand{\vecX}{{\bm X}}
\newcommand{\vecy}{{\bm y}}
\newcommand{\vecY}{{\bm Y}}
\newcommand{\vecZ}{{\bm Z}}
\newcommand{\vecz}{{\bm z}}
\newcommand{\vecU}{{\bm U}}
\newcommand{\vecW}{{\bm W}}

\newcommand{\bfxi}{\bm \xi}
\newcommand{\bfmu}{\bm \mu}
\newcommand{\bfbeta}{\bm\beta}
\newcommand{\bftheta}{{\bm\theta}}
\newcommand{\bfTheta}{\bm\Theta}

\newcommand{\bfeta}{\bm\eta}

\newcommand{\bfeps}{\bm \epsilon}
\newcommand{\bhbeta}{\widehat{\bm\beta}}
\newcommand{\bfvareps}{\bm \varepsilon}

\newcommand{\bfvartheta}{\bm\vartheta}
\newcommand{\bfGamma}{\mathbb{\Gamma}}
\newcommand{\matGamma}{\mathbb{\Gamma}}
\newcommand{\bfSigma}{\mathbb{\Sigma}}

\newcommand{\matxi}{\mathbb{\Xi}}

\newcommand{\EE}{\mathbb E}
\newcommand{\PP}{\mathbb P}
\newcommand{\EEE}{\operatorname{E}}
\newcommand{\PPP}{\operatorname{P}}
\newcommand{\Var}{{\mbox{Var\,}}}
\newcommand{\Cov}{{\mbox{Cov\,}}}

\newcommand{\diag}{{\mbox{diag\,}}}

\newcommand{\id}{\operatorname{id}}

\newcommand{\vecsf}{{\bm r}}

\usepackage{subfigure}
\usepackage{mathtools} 
\usepackage{bm}
\usepackage{dsfont}
\usepackage{color}
\usepackage{amsthm}

\usepackage{float}                 

\begin{document}

\begin{center}
	\begin{minipage}{.8\textwidth}
		\centering 
		\LARGE { Consistency of Extreme Learning Machines  and Regression under Non-Stationarity and Dependence for ML-Enhanced Moving Objects} \\[0.5cm]
		
		\normalsize
		\textsc{Ansgar Steland}\\[0.1cm]
		RWTH Aachen University,\\
		Aachen, Germany\\
		Email: \verb+steland@stochastik.rwth-aachen.de+
		
	\end{minipage}
\end{center}

\begin{abstract}
  Supervised learning by extreme learning machines resp. neural networks with random weights is studied under a non-stationary spatial-temporal sampling design which especially addresses settings where an autonomous object moving in a non-stationary spatial environment collects and analyzes data. The stochastic model especially allows for spatial heterogeneity and weak dependence. As efficient and computationally cheap learning methods (unconstrained) least squares, ridge regression and $ \ell_s $-penalized least squares (including the LASSO) are studied. Consistency and asymptotic normality of the least squares and ridge regression estimates as well as corresponding consistency results for the $ \ell_s $-penalty are shown under weak conditions. The results also cover bounds for the sample squared predicition error. 
\end{abstract}

\textit{Keywords:} Consistency, deep learning, extreme learning machine, high-dimensional data, lasso, machine learning, penalized estimation, ridge regression, spatial statistics

\section{Introduction}

A  growing number of instances of machine learning takes place in real time by  (autonomous) objects with limited computing resources, which move in a spatial domain and collect target and explanatory observations from sensors and other sources. Examples of this data sampling and learning setting are (electric) cars using data from sensors (and cameras) for autonomous driving or optimizing usage of attached photovoltaic modules, pedestrians with wearable sensors or flying objects (drones, planes) collecting environmental data, but also tiny sensors for drug monitoring in patients and future applications such as devices for focused drug application in cancer treatment. Such a setting leads to nonstationary dependent samples, for example due to possible interactions with the environment resulting in spatial-temporal dependencies and heterogeneous dispersion. 

For supervised learning, i.e. regression problems, artificial neural networks (including deep learners) with random weights, dating back to \cite{SchmidtEtAl1992} and studied as extreme learning machines (ELMs),  \cite{Huang2004,Huang2006}, random Fourier feature bases and random kitchen sinks, \cite{RahimiRecht2007,RahimiRecht2008a,RahimiRecht2008b}, are an attractive approach due to their extremely fast learning by reducing the optimization part to the output layer, where random features calculated by hidden layers resp. nonlinear feature functions are linearly combined. That last step is often implemented by least squares or ridge regression. The aim and contribution of this paper is to establish consistency results, bounds for the sample prediction error and asymptotic normality results under a nonstationary model for the errors covering the above setting of a moving object going beyond the classical i.i.d. framework. The estimation methods under investigation are least squares estimation, ridge regression (Tikhonov regularization) and $ \ell_s $-penalized least squares for $ 1 \le s \le 2 $ (thus covering the LASSO) are studied. Due to the close relationship of machine learning with ELMs and random features with multivariate regression, the results contribute to both areas.

Our approach to model data acquisition and supervised least squares learning by a moving object is as follows. It is assumed that the object moves through a $q$-dimensional spatial domain, typically $ q = 3 $, and observes at discrete time instants $ t = 1, \ldots, T $ 
input vectors $ \vecz_t $ and $d$-dimensional responses $ \vecY_t $. The inputs $ \vecz_t $ are either used as regressors or they are processed by a neural network with random weights (ELM) yielding (random) features $ \vecx_t $ collected in a $ T \times p $ data matrix $ \matX_T $. The data is then used to train an ELM respectively to fit a regression model using the estimation methods mentioned above.

The proposed model is defined in terms of an underlying continuous-domain spatial-temporal random field which models the (maximal) set of space-time trajectories of an object travelling through the spatial landscape. The random field is given in terms of underlying (unobservable) factor processes.  The model can be seen as an extension of space-time models obtained by smoothing underlying spatial Gaussian processes as studied in spatial statistics, see  \cite{Higdon2002} and the references given there. It allows, however, for non-normal data,  a flexible dependence structure with respect to time and nonlinearity and thus goes beyond \cite{Higdon2002} and similar models. The model is also related to (nonlinear) factor models, which have been extensively studied in time series analysis as well as for spatial data,  see, e.g., \cite{RenBanerjee2013} for hierarchical factor models using a low-rank predictive approach, and the references given there. But the model used in this paper differs  from classical factor models in that it introduces a weighting scheme and allows for a nonlinear relationship given by a Lipschitz continuous mapping. 
In this work, the factors are used to model non-stationarity and possible interactions of the object with its (local) environment. They may represent physical objects of the landscape, such as sources of emission when considering environmental monitoring, or virtual factors (without physical meaning), used to span a rich class of non-stationary models.  The factors are located at certain (fixed) spatial positions, either on a regular grid or at arbitrary positions. There is no need to know these locations, neither it is the goal to estimate them or the model. We introduce weights depending on the distance of the moving object to those factor locations determining the influence of each factor. Therefore, when the factors are arranged on a grid spanning the spatial domain, the resulting model provides a  flexible framework to allow for spatially inhomogenous correlations, and this approach has the advantage that it also allows to interpret the model as a means to explain
 the stochastic relationship between the moving object and the factors (sources).

The paper is organized as follows. In Section~\ref{Sec: Regression ML} the multivariate supervised learning problem is defined and ELMs are briefly reviewed and put in the framework of our results. We further show that under general conditions the theoretical solution of the least squares problem  corresponding to the first $ T $ observations converges to a well defined quantity. 
Section~\ref{Sec: Model Assumptions} introduces and explains the proposed model in greater detail and gives the main assumptions. The asymptotic theory is presented in Section~\ref{Sec: Theory}, whereas proofs are provided in Section~\ref{Sec: Proofs}.

\section{Problem Setup and Machine Learning for Representable Learning Problems}
\label{Sec: Regression ML}

The problem of supervised regression learning for an object moving in a stochastic environment as outlined in the introduction can be formulated as follows: The moving object observes data at discrete time points $ t = 1, \ldots, T $ with associated spatial locations $ \vecs_1, \ldots, \vecs_T $,
	\[ (\vecY_t^\top, \vecZ_t^\top, \vecs_t^\top)^\top, \qquad  t = 1, \ldots, T, \]
defined on a common probability space. 
The inputs $ \vecZ_t $ are either processed by a randomized neural network with linear output activation function  or analyzed by a multivariate linear regresssion model using regressors $ \vecX_t \in \R^p $, $p \in \N $, which may be functions of $ \vecZ_t $. We assume a signal plus noise model given $ \vecX_t = \vecx_t $, i.e.
\[
	\vecY_t = \bfmu_t + \bfeps_t, \qquad t = 1, \ldots, T,
\]
where $ \bfmu_t = \bfmu_t( \vecs_t, \vecx_t ) \in \R^d $ and $ \bfeps_t = \bfeps_t( \vecs_t ) $ is a $d$-variate mean zero spatial-temporal noise introduced and discussed in Section~\ref{Sec: Model Assumptions}. Throughout the paper the $ \nu$th coordinate of the error process $ \bfeps_t $ will be denoted by  $ \epsilon_t^{(\nu)} $, $ \nu = 1, \ldots, d  $. $ \bfmu_t( \vecs, \vecx ) $ is an underlying (conditional) mean function which is only observed and modeled at the relevant trajectory $ (\vecs_1, \vecx_1), \ldots, (\vecs_T, \vecs_T) $ by an ELM or a regression model. Extreme learning machines as well as multivariate regression models can be formulated by modeling the $ d $-dimensional response vectors $ \vecY_t$ given $p$-dimensional explanatory variables $ \vecx_t $, $ p, d \in \N $, by 
\begin{equation}
\label{RegressionModel}
\vecY_t^\top = \vecx_t^\top \matB + \bfeps_t^\top, \qquad t = 1, \ldots, T,
\end{equation}
where $ \vecY_t $ is the $d$-dimensional response vector, 
$ \matB = [ \bfbeta_1, \ldots, \bfbeta_{d} ] $ is an unknown $ p \times d $ matrix of regression parameters and $ \bfeps_t $ the $d$-vector of mean zero errors. In case of an ELM $ \vecx_t  $ is the output of the last hidden layer, as reviewed below. 


The ridge regression estimator $ \wh{\matB}_T^{(R)}  = ( \wh{\bfbeta}_{T,1}^{(R)}, \dots, \wh{\bfbeta}_{T\nu}^{(R)} ) $ for $ \matB$  minimizes
\[
  \tilde\matB \mapsto \| \matY_T - \matX_T \tilde\matB \|_2^2 + \lambda \| \tilde\matB \|_2^2
  = \sum_{\nu=1}^d \| \vecY_T^{(\nu)} - \matX_T \tilde\vecb_\nu \|_2^2 + \frac{\lambda}{2} \sum_{\nu=1}^d \| \tilde\vecb \|_2^2, \quad \tilde\matB \in \R^{p \times d},
\]
for some regularization parameter $ \lambda \ge 0 $. For $ \lambda = 0 $ one obtains classical least squares. Here, $ \matX _T = ( \vecx_1, \ldots, \vecx_T )^\top \in \R^{T \times p} $ is assumed to be of full rank,  $ \matY_T = (  \vecY_1, \ldots, \vecY_T )^\top $ is the $ T \times d $ matrix of the responses and $ \tilde\matB = [\tilde\vecb_1, \dots, \tilde\vecb_d] $. For some  common activation functions \cite{LiuEtAl2015} have shown explicitly that the output matrix $ \matX_T $ of the last hidden layer has full rank. The estimator is given by
\[
	\wh{\matB}_T^{(R)} = (\matX_T^\top \matX_T  + \lambda \matid )^{-1} \matX_T^\top \matY_T, 
\]
where $ \matid $ denotes the $p \times p $ identity matrix, and thus can be calculated column-by-column by simply solving linear equations. Here the $\nu$th column corresponds to the model
\begin{equation}
\label{ModelNu}
Y_t^{(\nu)} = \vecx_t^\top \bfbeta_\nu + \epsilon_t^{(\nu)}, \qquad  t = 1, \ldots, T,  \nu = 1, \ldots, d,
\end{equation}
where $ \epsilon_t^{(\nu)} $ denotes the $ \nu $th coordinate of the multivariate noise $ \bfeps_t $, which is fitted by minimizing $ \| \vecY_T^{(\nu)} - \matX_T \tilde{\vecb}_\nu \|_2^2 + \frac{\lambda}{2} \| \tilde{\vecb} \|_2^2 $, $ \tilde{\vecb} \in \R^p $, leading to solution $  \wh{\bfbeta}_{T\nu}^{(R)} $. Throughout the paper the least squares solution is denoted $ \wh{\bfbeta}_T^{(\nu)} $.

 It is well known that the ridge estimator shrinks the estimates towards the grand mean and, if the models contains an intercept term, usually by putting $ x_{t1} = 1 $ for all $t $, then the coefficients of the regressors are shrunken towards zero. 
It converges to $ \overline{\matY}_T = T^{-1} \sum_{t=1}^T \vecY_t $ if $ \lambda \to \infty $.  

Ridge regression is attractive due to its low computational costs and thus is often used to train randomized neural networks. A LASSO penality, however, has the advantage that it automatically selects the best fitting neurons from the last hidden layers which represent the final features of the net. These best fitting neurons are then linearly combined at the output layer.

Therefore, we also study the $ \ell_s $-norm penalized least squares problem, $ 1 \le s \le 2 $, in the following specification: The $ \nu$th model (\ref{ModelNu}) is fitted by minimizing 
\[
\tilde\vecb \mapsto \| \vecY_T^{(\nu)} - \matX_T \tilde \vecb \|_2^2+ \pen_\lambda( \tilde\vecb ), \qquad \tilde\vecb \in \R^p,
\]
where the penalty term is given by
\[
\pen_\lambda( \tilde\vecb )  = \lambda \| \tilde\vecb \|_s, \qquad \tilde\vecb \in \R^p. 
\]
For $ s = 1 $ we obtain the LASSO estimate for $ \bfbeta_\nu $. The corresponding $d$-dimensional separable $ \ell_s $-penalized least squares problem is to minimize
\[
  \tilde\matB \mapsto \| \matY_T - \matX_T \tilde\matB \|_2^2 + \lambda \sum_{\nu=1}^d \| \tilde\vecb \|_s, \quad \tilde\matB \in \R^{p \times d}.
\]
The LASSO of \cite{Tib1996} dates back to \cite{ChenDonoho1994}, where it has been called basis pursuit, and has been studied quite extensively for independent data. For fixed $p$ \cite{KnightFu2000} provided consistency and asymptotic distributions. High-dimensional settings were studied by \cite{MeinshausenBuehlmann2006},  \cite{MeinshausenYu2009} and \cite{BuehlmannGeer2011}.

\subsection{Extreme learning machine and randomized neural networks}
\label{Sec: ELM}

Extreme learning machines are a popular approach to machine learning based on artifical neural networks, see \cite{HuangWang-2011} for a review, belonging to the class of randomized neutral networks. They are very fast to train, rank among the best classifiers in large scale comparison studies using real data, \cite{JMLR:v15:delgado14a}, and therefore have been chosen as a benchmark classifier for the extended MNIST data set of handwritten characters, \cite{EMNIST2017}.
Originally developed for supervised learning as studied here, \cite{Huang2004,Huang2006}, they have been applied to numerous problems and  extended to various other learning problems including semi- and unsupervised learning problems, see \cite{HuangSongGuptaWu2014}, and multi layer networks, \cite{TangDengHuang2016}. Here we focus on single hidden layer feedforward networks and multi-layered feedforward nets for deep learning. In \cite{LiuEtAl2015} it has been shown that, for i.i.d. learning data, the generalization bounds known for feedforward neural networks, see \cite{GyoerfiKohlerEtAl2002}, essentially carry over to ELMs for nice activation functions.


 An extreme learning machine resp. neural networks with random weights, \cite{SchmidtKraaijveldDuin1992} and \cite{Huang2004}, with $q$ input nodes, $p$ hidden neurons and $d$-dimensional output in the form of a single hidden layer feedforward network with activation function $g$ computes the output of the $j$th neuron of the hidden layer for an input vector $ \vecz_t \in \R^q $ by
\begin{equation}
\label{ELM1}
x_{tj} = g\left( b_j + \vecw_{j}^\top \vecz_t  \right), \qquad j = 1, \ldots, p, \ t = 1, \ldots, T,
\end{equation}
where $ \vecw_j \in \R^q  $ are the weights connecting the input nodes and the hidden units, and  $ b_j \in \R $ are  bias terms, $ j = 1, \ldots, p $. Typical choices for the activation function $ g(u) $ are squashing functions such as the classical sigmoid function $ 1/(1+e^{-u}) $, the ReLU function $ g(u) = \max(0, u)  $ or the algebraic polynomial $ g(u) = u^s $.  

A multilayer (deep) learning neural network is given by a composition (concatenation) of, say, $r$, hidden layers with activation functions $ g_k : \R^{n_{k-1}} \to \R^{n_k} $, weighting matrices $ \matW^{(k)} = ( \vecw_1^{(k)}, \ldots, \vecw_{n_k}^{(k)} )^\top \in \R^{n_k \times n_{k-1}} $, bias terms $ \vecb_{k} = (b_{1k}, \ldots, b_{n_k,k})  \in \R^{n_k} $, satisfying $ n_0 = q $ and $ n_r = p $, such that the $j$th output of the $k$th layer is computed recursively as  
\begin{align*}
  x_{tj}^{(1)} &= g_1( b_{j1} + \vecw_{j}^{(1)}{}^\top \vecz_t ), \qquad j = 1, \ldots, n_1, \\
  x_{tj}^{(k)} &= g_k ( b_{jk} + \vecw_{j}^{(k)}{}^\top \vecx_t^{(k-1)}  ), \qquad j = 1, \ldots, n_k,  \ k = 2, \ldots, r,
\end{align*}
with $ \vecx_t^{(k)} = ( x_{t1}^{(k)}, \ldots, x_{tn_k}^{(k)} )^\top $, for $k = 2, \ldots, r $. 
We can write these equations as 
\[
\vecx_t^{(k)} =  \vecg_k^{(\matW^{(k)}, \vecb_k)}( \vecx_t^{(k-1)} ) = g_k( \vecb_k + \matW^{(k)}  \vecx_t^{(k-1)} ), 
\]
where  for a real-valued function $f $ defined on reals and a vector $ \vecx $ the expression $ f( \vecx ) $ is understood component-wise. Thus, the output $ \vecx_t $ of the $r$th hidden layer for an input $ \vecz_t $ is given by
\[
  \vecx_t = \vecg_r^{(\matW^{(r)}, \vecb_r)} \circ \cdots \circ \vecg_1^{(\matW^{(1)}, \vecb_1)}( \vecz_t )
\]
with $j$th coordinate
\begin{equation}
\label{DeepELM}
  x_{tj} = g_r^{(\matW_j^{(r)}, b_{jr})} \circ  \vecg_{r-1}^{(\matW^{(r-1)}, \vecb_{r-1})} \circ \cdots \circ  \vecg_1^{(\matW^{(1)}, \vecb_1)}( \vecz_t  ),
\end{equation}
where $ \matW_j^{(r)} $ denotes the $j$th row of $ \matW^{(r )} $.

Deep learners are often not fully connected and use specialized (fixed) topologies. For example, convolutional layers and pooling layers are heavily used in deep learners processing image data. If the $k$th layer is a convolutional layer, the input $ \vecx_t^{(k-1)} $ (where $ \vecx_t^{(0)} = \vecz_t $) is structured in $n_k $ subvectors of dimensions $ n_{k1}, \ldots, n_{k,n_k} $, such as $ \vecx_t^{(k-1)}{}^\top = (  \vecx_{t1}^{(k-1)}{}^\top, \ldots, \vecx_{t,n_k}^{(k-1)}{}^\top ) $, and the $j$th node  computes $n_k $ convolutions of the $j$th subvector with weights $ \vecw_{tj}^{(k)} \in \R^{n_{k,n_j}} $,
\[
  x_{tj}^{(k)} = \vecw_{tj}^{(k)}{}^\top \vecx_{tj}^{(k-1)}{}, \qquad j = 1, \ldots, n_k,
\]
e.g. to compute local (image) features. Such convolutional layers are typically followed by pooling layers computing the mean or max feature activation to aggregate those convolved features over certain regions.

Let $ \vecx_t = ( x_{t1}, \ldots, x_{tp} )^\top $ be the output of the last hidden layer. The output layer now processes these values using the linear activation function, such that the output $ o_t^{(\nu)} \in \R $ of the $ \nu$th output neuron is given by
\begin{equation}
\label{ELM2}
	o_t^{(\nu)} = \vecx_t^\top \bfbeta_{\nu}, \qquad t = 1, \ldots, T,
\end{equation}
for $ p $-dimensional weighting vectors $ \bfbeta_{\nu} $, $ \nu = 1, \ldots, d $.  Here we assume, for simplicity of presentation, that the first elements of the $ \vecx_t $'s are ones, such that the bias term is absorbed into $ \bfbeta_{\nu} $. The $d$-dimensional net output for an input sample $ (\vecY_t, \vecz_t) $, $ t = 1, \dots, T $, can be written as
\[
  \veco_t^\top = (o_t^{(1)}, \ldots, o_t^{(d)} ) = \vecx_t^\top \matB, \qquad t = 1, \ldots, T,
\]
when defining the  $ p \times d $ weighting matrix  $ \matB = ( \bfbeta_1, \ldots, \bfbeta_d ) $, and by stacking these $T$ rows we may write
\[
  \matO_T = \matX_T \matB,
\]
where $ \matO_T = ( \veco_1, \ldots, \veco_T )^\top $ is the $ T \times d $ output matrix.  The output layer of a single hidden layer network is parameterized by the $ p \times d $ weighting matrix  $ \matB $, and the whole network is parameterized by $ \bfvartheta = (\vecb, \matW, \matB ) $ where $ \vecb = (b_1, \ldots, b_p) $ and $ \matW = (\vecw_1, \ldots, \vecw_p ) \in \R^{q \times p} $. A multilayer network is parameterized, for a fixed topology of the net, by $ \bfvartheta = ( \matW^{(j)}, \vecb_j : j = 1, \ldots, r ) $.

An ELM uses random weights for all layers except the output layer and thus optimizes only the output weights. Usually, independent and identically distributed random biases and random weights are generated to connect the inputs and the hidden neurons. For example $ b_j \sim \calN(0, \sigma_b^2 ) $ and $ \vecW^{(j)} \sim \calN(\vecnull, \sigma_w^2 \matid ) $, $ j = 1, \ldots, r $, for $ \sigma_b^2, \sigma_w^2  > 0 $. See \cite{Dudek2019} for some tuned choices of the probability distributions for the weights depending on the activation function. As the law of the weights is chosen by the user, one can easily implement networks with sparse weights. The random choice of the connection weights generates samples of neuron outputs $ (x_{1j}, \ldots, x_{Tj}) $, $ j = 1, \ldots, p, $ which are independent respectively conditionally independent given the inputs $ \vecz_t $, since the weights of each neuron are generated independently at random. 


The weights of the output layer of an ELM are optimized by minimizing a (penalized) least squares criterion, as proposed by \cite{Huang2004} without refering to a stochastic regression model. But as well known, this is equivalent to a (penalized) least squares regression of the model
\begin{equation} 
	\label{AuxModel}
	\matY_T = \matO_T + \mathbb{e}_T,
\end{equation} 
where  $ \matY_T = ( \vecY_T^{(1)}, \ldots, \vecY_T^{(d)} ) $ is the $ T \times d $ matrix of target value, $ \vecY_T^{(\nu)} = ( Y_{1}^{(\nu)}, \ldots, Y_{T}^{(\nu)} )^\top $ is the available sample of the $ \nu$th target variable and $ \mathbb{e}_T $ is a $ T \times d $ random matrix of mean zero errors.

I.i.d. sampling of the random weights is a common approach, but there are two reason to allow for non-i.i.d. sample, especially dependent sampling:

\begin{remark} Attaching random $ \ell_1 $-weights is attractive to achieve sparsity of the network weights and can be easily achieved by using a Dirichlet distribution. 
\end{remark}

\begin{remark}
We conjecture that dependent sampling may improve the performance considerably. For example, if the the data splits into two groups depending on the value of one input variable, generating at the last but one layer groups of neurons which generate multivariate Gaussian distributions 
strongly concentrated around different subspaces may simplify the separation of the subsamples at the output layer. 
\end{remark}

\subsection{Representable learning problems under nonstationarity}
\label{ReprLProb}

To simplify presentation, let us fix $ \nu \in \{ 1, \dots, d \} $. Assume for a moment that  $ ( \vecX, Y^{(\nu)}), ( \vecX_t, Y_t^{(\nu)} ), $ $t = 1, 2, \dots $, is a strictly stationary sequence defined on a common probability space $ ( \Omega, \calF, \PPP ) $. Denote the expection w.r.t $ \PPP $ by $ \EEE $. The solution, $ \bfbeta_\nu $, of the $ \nu $th theoretical least squares problem, $ \EEE (Y_t^{(\nu)} - \vecX_t^\top \vecb )^2 \to \min_\vecb $, is independent of time,  satisfies $ [ \EEE( \vecX \vecX^\top ) ] \bfbeta_\nu = \EEE( Y \vecX ) $, and is therefore given by $ \bfbeta_\nu = [ \EEE( \vecX \vecX^\top ) ]^{-1} \EEE( Y \vecX ) $ if the matrix of second order moments $  \EEE( \vecX \vecX^\top ) $ is regular. But for nonstationary data, as we shall study here, we need to define an appropriate framework where the theoretical solutions converge to a well defined quantity as the sample size and the dimension grow. The following discussion is not required for the results of the following sections, but it is insightful and also clarifies the role of the true parameters $ \bfbeta_\nu $, $ \nu \ge 1 $.

Let us drop the stationarity assumption. Instead, we essentially assume ergodicity in the form of the strong law of large numbers (SLLN). These conditions are sufficient and avoid  convergence of the process itself. Precisely, we assume that for all $ \nu \ge 1 $

\begin{align}
	\label{AsConv}
& \text{$ \{ \vecX_t Y_t^{(\nu)} \}$  satisfies the SLLN and }\text{$ |  \vecX_t Y_t^{(\nu)} | \le W, t \ge 1, $ for some  $ W $ with $ \EE_{Y|X}(W) < \infty $},  
\end{align}
Here $ \EEE_{Y|X} $ denotes the conditional expectation given  $\vecX_t, t \ge 1 $.  Recall that a sufficient condition for the second assumption is $ \sup_{t \ge 1} \EE | \vecX_t Y_t^{(\nu)} |^{1+\delta} $ for some $ \delta > 0 $. It can also be replaced by uniform integrability.

For a sample $ \bfxi_1, \ldots, \bfxi_T $ of random vectors denote the associated empirical measure   by $ \PPP_T(A) = \frac{1}{T} \sum_{t=1}^T \eins_A( \bfxi_t ) $. $ \PPP_T $ defines an expectation operator for random vectors $ \vecZ $ and vector- or matrix-valued mappings $ h $ such that $ \int h( \vecu ) \, d \PPP_T( \vecu ) $ exists elementwise, namely 
\[
 \EEE_T( h( \vecZ ) ) = \int \vecu \, d\PPP_T( \vecu ) = \frac{1}{T} \sum_{t=1}^T h (\bfxi_t),
\] 
i.e., the mean of $ h( \vecZ ) $ under $ \vecZ \sim \PPP_T $. In addition, introduce the operator
\begin{equation}
\label{DefExp}
\EEE_T^*( h (\vecZ) ) := \EEE_{Y|X}( \EEE_T( h( \vecZ ) ) ) = \frac{1}{T} \sum_{t=1}^T \EEE_{Y|X}( h( \bfxi_t) ),
\end{equation}
the average of conditional expectations.
Now put $ \bfxi_t = (Y_t^{(\nu)}, \vecX_t ) $, $ t \ge 1 $.  $ \EEE_T^* $ serves as an appropriate analogon for $ \EEE $. Thus, consider the least squares criteria
\[ 
	\EEE_T^* (\vecY - \vecX^\top \vecb)^2 = \frac{1}{T} \sum_{t=1}^T \EEE_{Y|X}( Y_t^{(\nu)} - \vecX_t^\top \vecb )^2 
\] 
and the associate multivariate versions 
$
   \EEE_T^* \| \matY_T - \matX_T \matB \|_2^2 = \sum_{\nu=1}^d \EEE_T^* \| Y^{(\nu)} - \vecX^\top \bfbeta_\nu \|_2^2 $.
Differentiating\footnote{Use $ \frac{\partial \veca^\top \vecx}{\partial \vecx} = \veca^\top $ and observe that $ \vecx^\top \bfbeta_\nu \vecx = \vecx \vecx^\top \bfbeta_\nu $. } with respect to $ \vecb $ leads to the equation
\[
  [  \EEE_T^*( \vecX \vecX^\top ) ] \bfbeta_{T\nu} = \EEE_T^*( \vecX Y^{(\nu)} )
\]
for the optimal solution $  \bfbeta_{T\nu} $. If $ \EEE_T^*( \vecX \vecX^\top ) = T^{-1} \sum_{t=1}^T \vecX_t \vecX_t^\top $  is invertible, we obtain
\[
  \bfbeta_{T\nu} = [  \EEE_T^*( \vecX \vecX^\top ) ]^{-1} \EEE_T^*( \vecX Y^{(\nu)} ) = \left( \frac{1}{T} \sum_{t=1}^T \vecX_t \vecX_t^\top \right)^{-1} \left( \frac{1}{T} \sum_{t=1}^T \EEE_{Y|X} ( \vecX_t Y_t^{(\nu)} ) \right).
\]
The following result studies the asymptotic limit of these least squares problems naturally associated to a nonstationary sequence $ (\vecY_t^{(\nu)}, \vecX_t) $, as $ T, p \to \infty $.  Since at this point we are mainly interested in identifying appropriate limits, we impose somewhat top-level assumptions. Results and assumptions in  Section~\ref{Sec: Theory} are more tailored to the assumed non-stationary model of Section~\ref{Sec: Model Assumptions}. For a matrix $ \matA $ the operator norm associated to the vector $2$-norm $ \| \cdot \|_2 $ is denoted by $ \| \matA \|_2 $.

\begin{theorem} 
	\label{GeneralConsistency}
	Supose that (\ref{AsConv}) holds. Then
	\begin{equation}
		\label{ConvGeneralElementwise}
		\EEE_T^*( \vecX Y^{(\nu)} )  \to  \EEE( \vecX Y^{(\nu)} )
	\end{equation}
	elementwise, as $ T, p \to \infty $. If (\ref{ConvGeneralElementwise}) holds w.r.t. $ \| \cdot \|_2 $, $ \| \EEE_T^*( \vecX \vecX^\top )  - \bfSigma_\vecx \|_2 = o(1) $ for some
	regular matrix $ \bfSigma_\vecx $ with $ \| \bfSigma_\vecx^{-1} \|_2 = O(1) $, and
	if additionally
	\begin{equation}
	\label{AsympCond}
	 \| \EEE( \vecX Y^{(\nu)} )  \|_2 = O(1),
	\end{equation}
	then
	\[
	  \| \bfbeta_{T\nu} - \bfbeta_\nu \|_2 \to 0,  \quad \bfbeta_\nu = \bfSigma_\vecx^{-1} \EEE( \vecX Y^{(\nu)} ),
	\]
	as $ T, p \to \infty $.
\end{theorem}

\begin{proof}
	Observe that $ \EEE_T^*( \vecX Y^{(\nu)} ) =  \EEE_{Y|X} \left( \frac{1}{T} \sum_{t=1}^T \vecX_t Y_t^{(\nu)} \right) $. Since $ T^{-1} \sum_{t=1}^T \vecX_t^{(\nu)} Y_t^{(\nu)} - \EEE( \vecX_t Y_t^{(\nu)} ) $ is dominated by $ 2 W $, the SLLN and dominated convergence yield
	\[
	  \EEE_T^*( \vecX Y^{(\nu)} ) - \EEE(  \EEE_T^*( \vecX Y^{(\nu)} )  )
	  = \EEE_{Y|X} \left(  \frac{1}{T} \sum_{t=1}^T [ \vecX_t Y_t^{(\nu)} - \EEE( \vecX_t Y_t^{(\nu)} ) ] \right) \to 0, \quad a.s.,
	\]
	as $ T, p \to \infty $, and using $ \EEE = \EEE \EEE_{Y|X} $ we also obtain
	\[
	 \EEE(  \EEE_T^*( \vecX Y^{(\nu)} )  ) - \EEE( \vecX Y^{(\nu)} ) 
	 	= \EEE\left(  \frac{1}{T} \sum_{t=1}^T \vecX_t Y_t^{(\nu)} -  \EEE( \vecX Y^{(\nu)} ) \right) 
	 \xrightarrow[T\to\infty]{} 0,
	\]
	if $ T, p \to \infty $. Next write 
	\begin{align*} 
	& [\EEE_T^*( \vecX \vecX^\top )]^{-1} \EEE_T^*( \vecX Y^{(\nu)} ) 
	- \bfSigma_\vecx^{-1}  \EEE( \vecX Y^{(\nu)} ) \\
	& \quad =
	[\EEE_T^*( \vecX \vecX^\top )]^{-1} ( \EEE_T^*( \vecX Y^{(\nu)} ) - \EEE( \vecX Y^{(\nu)} ) ) + (  [\EEE_T^*( \vecX \vecX^\top )]^{-1} - \bfSigma_\vecx^{-1} ) \EEE( \vecX Y^{(\nu)} ) ).
	\end{align*}
	We obtain
	\begin{align*}
	  & \| [\EEE_T^*( \vecX \vecX^\top )]^{-1} \EEE_T^*( \vecX Y^{(\nu)} ) 
	  - \bfSigma_\vecx^{-1}  \EEE( \vecX Y^{(\nu)} ) \|_2 \le R_{T1} + R_{T2}, 
	\end{align*}
	where
	\begin{align*}
	  R_{T1} & = \| [\EEE_T^*( \vecX \vecX^\top )]^{-1} ( \EEE_T^*( \vecX Y^{(\nu)} ) - \EEE( \vecX Y^{(\nu)} )  \|_2  \\
	  & \le \left( \| [\EEE_T^*( \vecX \vecX^\top )]^{-1} - \bfSigma_\vecx^{-1} \|_2 + \| \bfSigma_\vecx^{-1} \|_2 \right) \| \EEE_T^*( \vecX Y^{(\nu)} ) - \EEE( \vecX Y^{(\nu)} ) \|_2 \\
	  & = o(1), 
	 \end{align*}
and
	\[  R_{T2} = \|  (  [\EEE_T^*( \vecX \vecX^\top )]^{-1} - \bfSigma_\vecx^{-1} ) \EEE( \vecX Y^{(\nu)} ) ) \|_2 
	   \le \|  (  [\EEE_T^*( \vecX \vecX^\top )]^{-1} - \bfSigma_\vecx^{-1} ) \|_2  \| \EEE( \vecX Y^{(\nu)} )  \|_2 \
	   = o(1), \]
as 	$ T, p \to \infty $.
\end{proof}

Condition (\ref{AsympCond}), which will also play a role when studying the rigde estimator, can also be justified as follows. Observe that $ \vecX^\top \bfbeta_\nu = \vecX^\top \bfSigma_\vecx^{-1} \EEE( \vecX Y^{(\nu)} ) $ is the orthogonal projection of $ Y^{(\nu)} $ onto the span of the entries $ X_1, \dots, X_p $ of $ \vecX $ and $ \bfbeta_\nu $ are the corresponding coefficients\footnote{If $ Y^{(\nu)} = \vecX^\top \bfeta $ for some $ \bfeta \in \R^p $, then $ \vecX^\top \bfbeta_\nu = \vecX^\top \bfSigma_\vecx^{-1} \EEE( \vecX \vecX^\top \bfeta ) =\vecX^\top \bfeta$.}.
 For many problems it is reasonable to assume that the signal (target)  $ Y^{(\nu)} $ has a $L_2$-representation in terms of the variables  $ X_j $, $ j = 1, 2, \ldots  $, i.e.
\begin{equation}
\label{ReprEQ}
  Y^{(\nu)} = \sum_{j=1}^\infty \theta_j^{(\nu)} X_j + \epsilon^{(\nu)}, \qquad
  \epsilon^{(\nu)} \perp \text{span} \{  X_1, X_2, \ldots \},
\end{equation}
in mean square, where $ \theta_j^{(\nu)} $ is the $j$th element of $  [ ( \EEE(X_k X_j ) )_{j,k\ge 1} ]^{-1} (\EEE( Y^{(\nu)} X_j ))_{j=1}^\infty  $ provided that the infinite-dimensional matrix $ ( \EEE(X_k X_j ) )_{j,k\ge 1} $ has absolutely summable columns and defines a bijective mapping $ \R^\infty \to \R^\infty $. Thus, the solution $ \bfbeta_\nu = \bfSigma_\vecx^{-1} \EEE( \vecX Y^{(\nu)} ) $ can be seen as an approximation of $ \bftheta_{\nu,1:p} = ( \theta_1, \ldots, \theta_p ) $. For an orthonormal system we simply have
\[
  \theta_j^{(\nu)} = \EEE( Y^{(\nu)} X_j ), \qquad j \ge 1. 
\]
If the system $ ( X_j )_{j=1}^\infty $ spans $ L_2( \Omega, \calF, \PPP )$, then the term $ \epsilon^{(\nu)} $ in (\ref{ReprEQ}) vanishes, of course. It is natural to impose decay conditions on the coefficients $ \bftheta_\nu = ( \theta_j^{(\nu)} )_{j=1}^\infty $. For example,  $ \| \bftheta_\nu \|_{\ell_0} = s < \infty $ or an $ \ell_2 $-condition,
\begin{equation}
\label{SparseRepr}
 \theta_j^{(\nu)} = O( j^{-\gamma} ),
\end{equation}
for some $ \gamma > 1/2 $, such that the projection coefficients have finite $ \ell_2 $-norm, $ \sum_j ( \theta_j^{(\nu)} )^2 < \infty $. The $ L_2 $ representation of a signal $ Y^{(\nu)} $ implies that the limit of the  related least squares solution yields the first $p$ coefficients of the series representation.   

\begin{definition} $ Y^{(\nu)} $ is called {\em representable in terms of $ X_1, X_2, \ldots $}, if (\ref{ReprEQ}) holds. It is called  {\em $ \ell_2 $-representable} if (\ref{ReprEQ})  holds  and $ \| \bftheta_\nu \|_{\ell_2} < \infty $.
\end{definition}

We shall provide consistency results for the ridge estimator, see Theorem~\ref{THRIDGE}, under the assumption that
\begin{equation}
\label{IsSatisfied}
  \| \bfSigma_\vecx \bfbeta_\nu \|_2 = \| \EEE( \vecX Y^{(\nu)} ) \|_2 = O(1),
\end{equation}
which is a natural one in view of the above discussion, cf.  \eqref{AsympCond}. Contrary, the consistency results for the $ \ell_s $-penalized least squares solution requires a sparsity constraint on $ \bfbeta_\nu $.

\section{Model and assumptions}
\label{Sec: Model Assumptions}

For simplicity of presentation, we assume that the spatial domain is given by the unit interval $ [\bm 0, \bm 1] $ in the $ q $-dimensional space. 

\subsection{Non-stationary model}

We model the data  by a discretely sampled spatial-temporal process modeling a space-time environment. So, let us assume we are given a $ d $-dimensional process
\begin{equation}
\label{STProcess}
\vecY_t( \vecs, \vecx ) = \bfmu_t( \vecs, \vecx ) + \bfeps_t( \vecs ), \qquad \vecs \in [\bm 0, \bm 1], \vecx \in \R^p, t = 1, \ldots, T,
\end{equation}
where  $ \bfmu_t(  \vecs, \vecx ) \in \R^{d} $ is the mean of $ \vecY_t $, 
\begin{equation}
\label{ModelErrors}
  \bfeps_t( \vecs ) = \vecH( \vecU_t( \vecs ), \ldots, \vecU_{t-m}( \vecs ) ) -
    \EE \vecH( \vecU_t(\vecs), \ldots, \vecU_{t-m}( \vecs ) )
\end{equation}
for some $m \in \N_0 $ and a function $ \vecH : \R^{d\times m} \to \R^d $ are error terms , where 
\begin{equation}
	\label{LinFactorModel}
	\vecU_t(\vecs ) = \matPhi( \vecs ) \matF_t + \vecE_t, \qquad \vecs \in [\bm 0, \bm 1], \vecx \in \R^p, t = 1, \ldots, T
\end{equation}
with
\[
  \matPhi(\vecs) = \left( \phi_{\nu \ell} d( \vecs, \vecsf_\ell ) \right)_{1 \le \nu \le d \atop 1 \le \ell \le L},
  \qquad
  \matF_t = \left( \vecF_t( \vecsf_1 ), \ldots, \vecF_t( \vecsf_L ) \right)^\top.
\]
for (factor loading) coefficients $ \phi_{\nu \ell} \in \R $ and $ d $-variate mean zero factor processes
  $ \vecF_t( \vecsf_\ell ) $, $ \ell = 1, \ldots, L $, modeling the $L$ sources of randomness distributed over the spatial domain at sites $ \vecsf_1, \ldots, \vecsf_L \in [\bm 0, \bm 1] $. $ d : [\bm 0, \bm 1] \times [\bm 0, \bm 1] $ is a given function, such that $ d(\vecs, \vecsf_\ell ) $ models the effect of factor $ \ell $ on a moving object located at $ \vecs $,  and $ \vecE_t $ is the idiosyncratic noise with mean zero.  This means, at site $ \vecs $ and time $t $ the non-random spatial-temporal mean $ \bfmu_t( \vecs, \vecx ) $ is disturbed by a nonlinear noise process (dependent on $ \vecs $ and $t$). For $ m = 0 $ and $ \vecH = \id $ it is given by $ ( \sum_{\ell=1}^L \phi_{\nu \ell} d(\vecs, \vecsf_\ell ) \vecF_t( \vecsf_\ell )  )_{\nu=1}^d + \vecE_t $, a linear combination of spatially distributed factors plus a noise process common to all spatial sites, but  in general it may depend on $m$ lagged values in a nonlinear way.  The above random field is sampled at location $ \vecs_t $ yielding the error term
  \[
		\bfeps_t = \bfeps_t( \vecs_t ) = \vecH( \vecU_{t:m} ) - \EE \vecH( \vecU_{t:m} ), \qquad  t = 1, \ldots, T,
  \]
  where $ \vecU_{t:m} = ( \vecU_t, \ldots, \vecU_{t-m} ) $ with $ \vecU_t = \vecU_t( \vecs_t ) $.  Observe that this model is related to classical factor models but differs in several aspects: In a classical factor model $ \matPhi \vecf_t $ each coordinate uses the same $d$-variate factors $ \vecf_t $ but different weights (loadings) to combine them. In model (\ref{LinFactorModel}) each coordinate uses different factors and the loadings may depend on the distance to the spatial location of the factors. Further, by allowing for a nonlinear transformation $ \vecH $ model (\ref{ModelErrors}) covers for more complicated error models such as $ \bfeps_t = h( \vecU_{t-1} ) \vecU_t $ for some positive function $h$, such that the past value $ \vecU_{t-1} $ affects the scale of the time $t$ noise.

In model (\ref{LinFactorModel})  $ d(\vecu, \vecv) $, $ \vecu, \vecv \in [\bm 0, \bm 1] $, is a nonincreasing function of the distance $ \| \vecu  - \vecv \| $ taking values in the unit interval $ [0,1] $ and attaining its maximal value on the diagonal, i.e., $ d( \vecu, \vecu ) = 1 $ for all $ \vecu \in [\bm 0, \bm 1] $. Natural choices are the function
\[
d(\vecu, \vecv) = 1 -  \| \vecu  - \vecv \|/M, \qquad \vecu, \vecv \in [\bm 0, \bm 1],
\]
where $ M = \sup \{ \| \vecu - \vecv \| : \vecu, \vecv \in [\bm 0, \bm 1] \} $, or, for some positive constant $ c_d $,
\[
d(\vecu, \vecv) = \frac{1}{1+c_d\| \vecu  - \vecv \| }, \qquad \vecu, \vecv \in [\bm 0, \bm 1],
\] 
or the exponential 
\begin{equation}
\label{ExpD}
	d(\vecu, \vecv) = \exp( - a \| \vecu  - \vecv \|), 
\end{equation}
for a constant $ a \ge 0 $ determining the smallest weight, where $ \| \cdot \| $ denotes an arbitrary vector norm.

The factor processes in (\ref{LinFactorModel}) are assumed to be given by $ d $-variate linear processes
\[
  \vecF_t^{(\ell)}( \vecsf_\ell ) = \sum_{j=0}^\infty \matC_{j}^{(\ell)}{} \bfvareps_{t-j}^{(\ell)}, \qquad t = 1, \ldots, T, \ell = 1, \ldots, L,
\]
for $ d \times d $ coefficient matrices $ \matC_{j}^{(\ell)} = ( \vecc_{j}^{(\ell,1)}{}, \ldots, \vecc_{j}^{(\ell,d)}{} )^\top  $ with rows $ \vecc_{j}^{(\ell,\nu)}{}^\top $, $ \nu = 1, \ldots, d $. In addition,  the idiosyncratic noise is modelled as 
\[
  \vecE_t = \sum_{j=0}^\infty \matC_{j}^{(0)}{} \bfvareps_{t-j}^{(0)}, \qquad t = 1, \ldots, T.
\]
A simulated trajectory of the noise model (for $ d = 1 $) is depicted in Figure~\ref{Fig1}.

\begin{figure}
	\begin{center}
		\includegraphics[width=3.5cm]{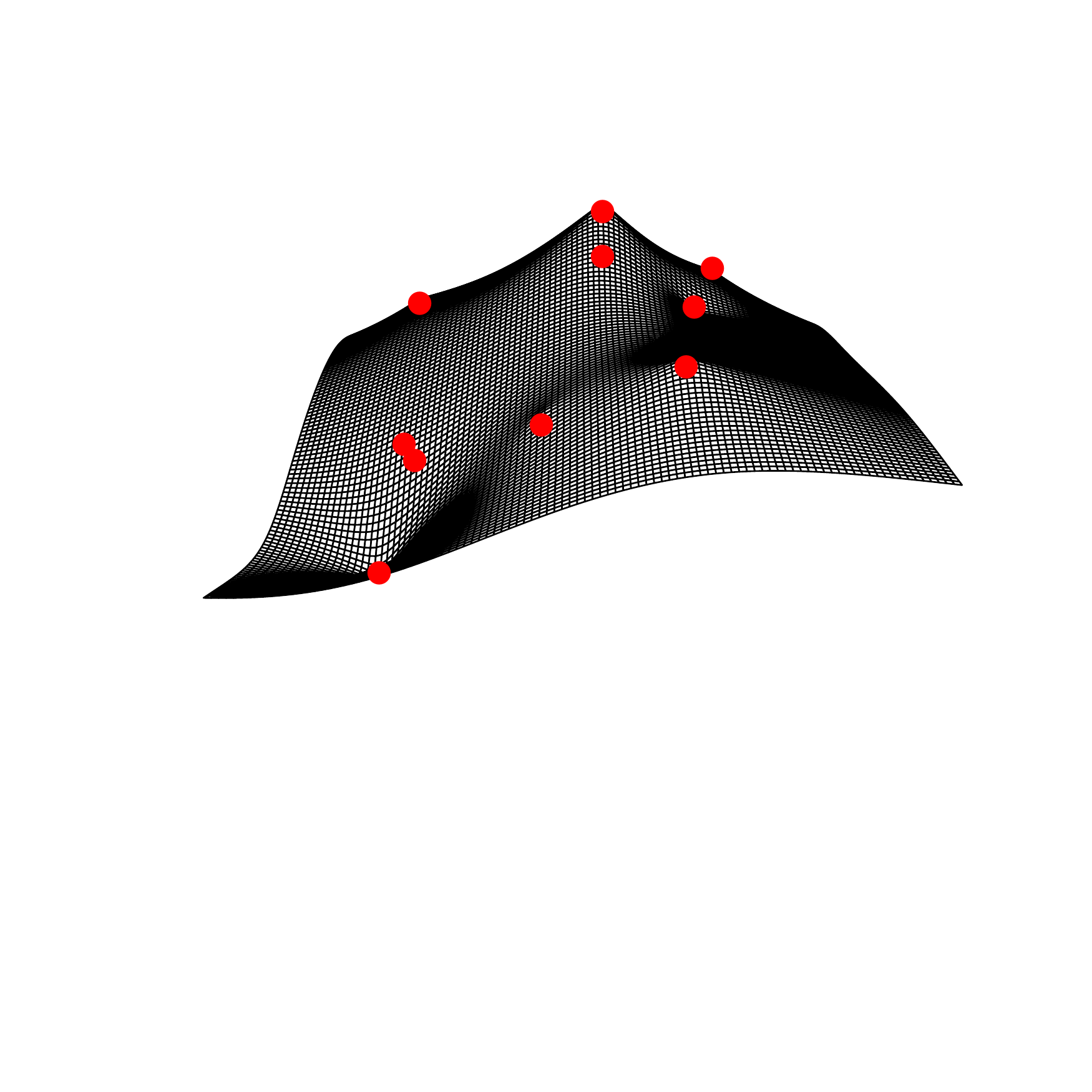}
		\includegraphics[width=3.5cm]{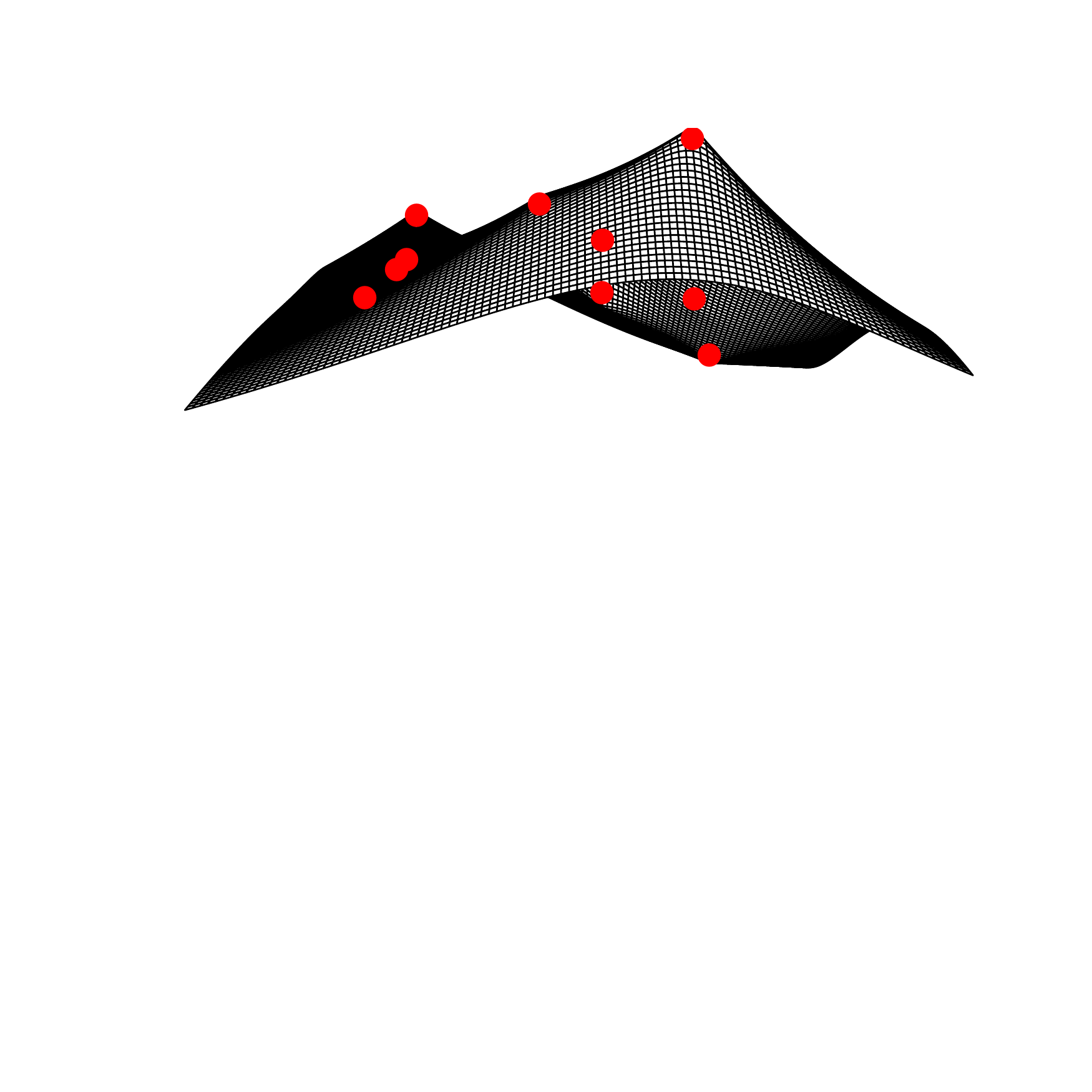}	\includegraphics[width=3.5cm]{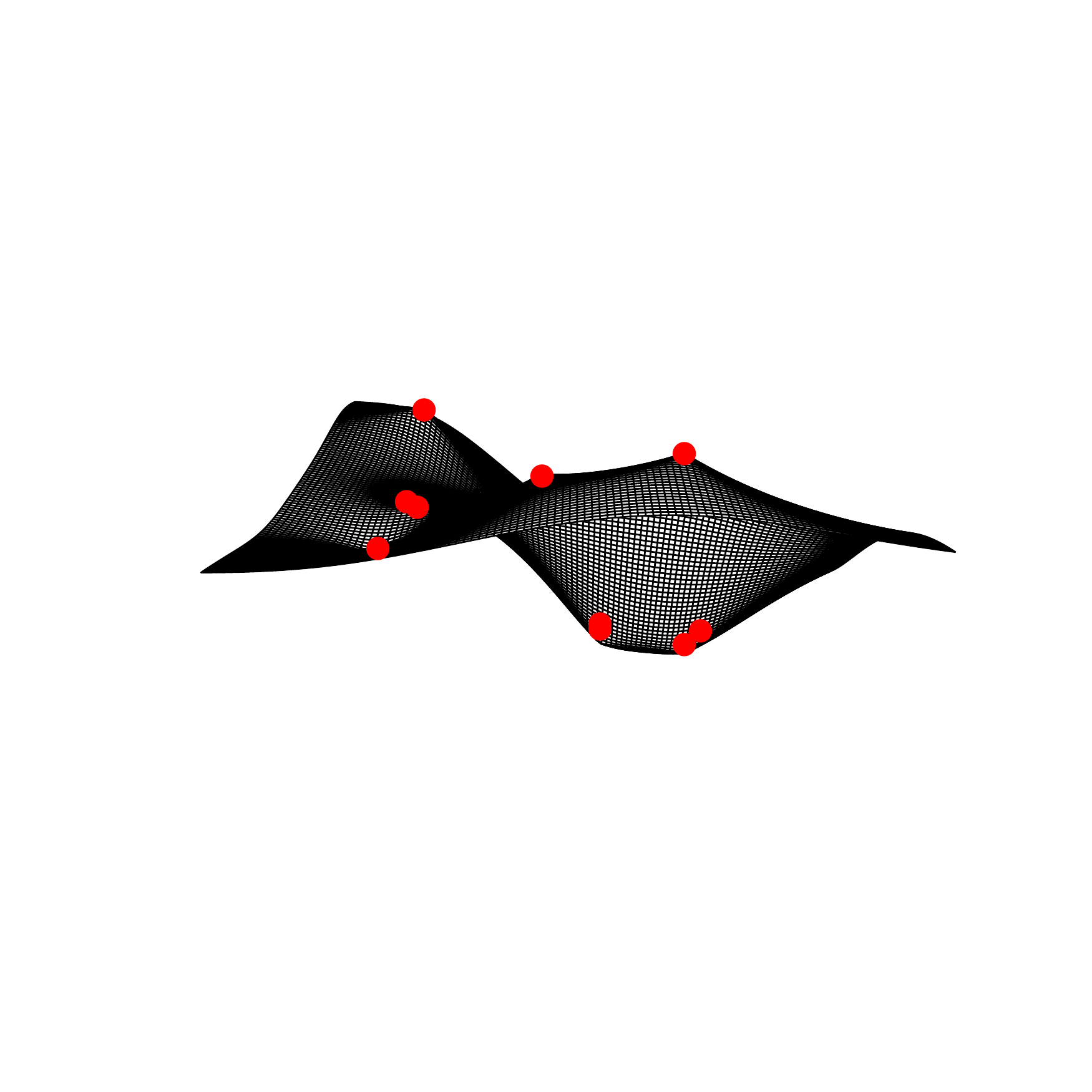}
		\includegraphics[width=3.5cm]{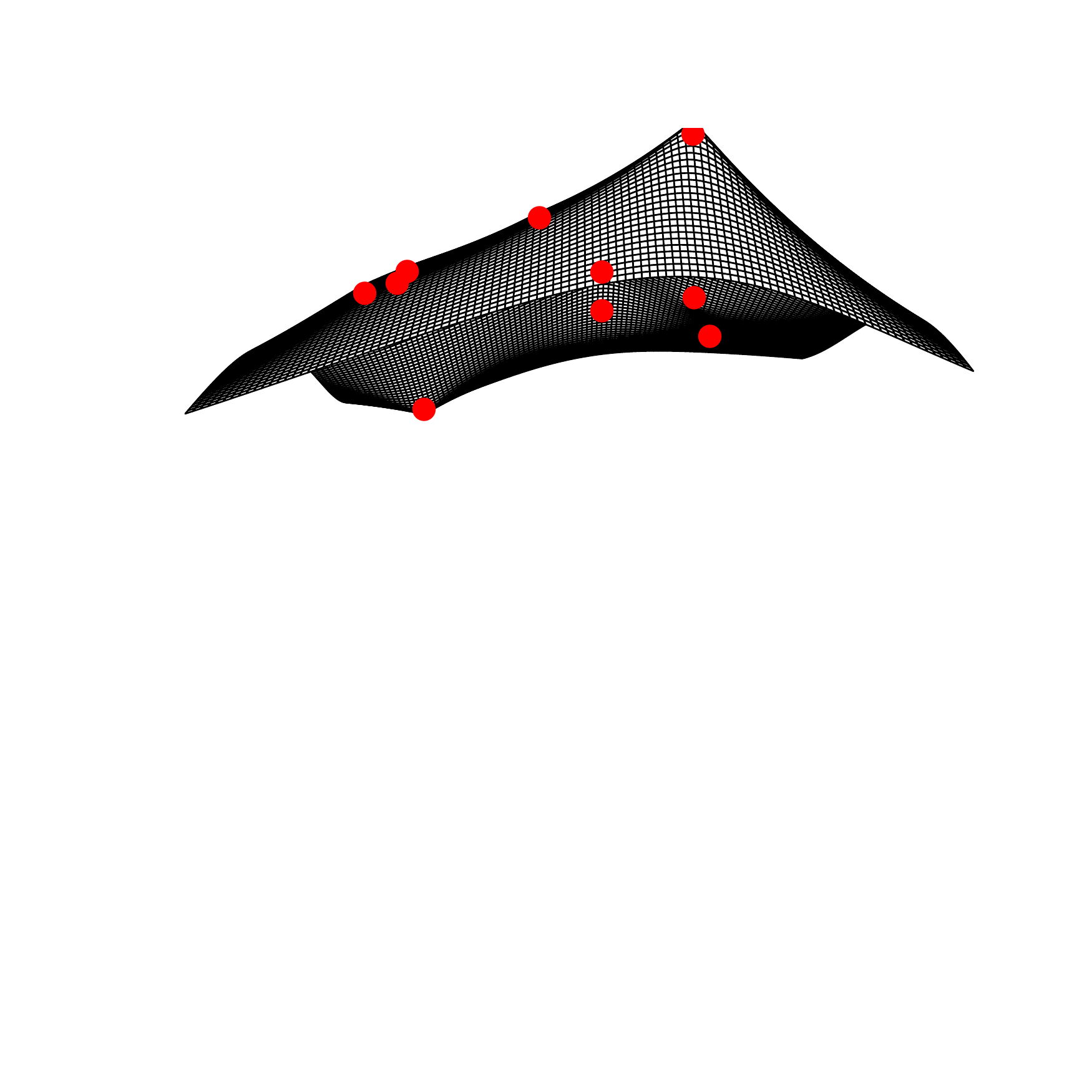}		
	\end{center}
	\caption{Four consecutive time points of a simulated trajectory of model (\ref{ModelErrors}) for $ \vecH = \id $ using the exponential function (\ref{ExpD}). The locations of the $L=10$ factor sources are marked by red points on the error surface.}
	\label{Fig1}
\end{figure}

\subsection{Assumptions on the errors}
\label{ModelErrors}

The assumptions on the error terms, the coefficients of the processes and the nonlinearity $ \vecH $ are as follows. It is assumed that
\begin{align}
\label{AssErrors1} 
&  \{ \bfvareps_t^{(\ell)} : t \in \Z \},  \ell = 1, \ldots, L,  \text{are $L$ independent sequences of independent random} \\ \nonumber & \text{vectors $ \bfvareps_t^{(\ell)} = \bfSigma_{\bfvareps}^{(\ell)}{}^{1/2} \vecu_t^{(\ell)}, t \in \Z, $ for $d$-dimensional mean zero random vectors $ \vecu_t^{(\ell)} $} \\ \nonumber
& \text{with independent coordinates,}
\end{align}
for symmetric $ d \times d $ matrices $ \bfSigma_{\bfvareps,t}^{(\ell)} > 0 $, and 
\begin{equation}
\label{AssErrors2}
  \sup_{t \in \Z} \max_{1 \le \ell \le L} \max_{1 \le \nu \le d} \EE | \varepsilon_t^{(\ell, \nu)} |^{2+\delta} < \infty, \qquad t \in \Z, \ell = 0, \ldots, L,
\end{equation}
for some $ \delta > 0 $, where $ \varepsilon_t^{(\ell, \nu)} $ denotes the $ \nu $th element of the random vector $ \bfvareps_t^{(\ell)} $. 

For the special case $ d = 1 $, we assume  
possibly inhomogenous but uniformly bounded variances $ \sigma_{t\ell}^2 = \EE( \varepsilon_t^{(\ell)}  )^2 $, i.e.,
\begin{equation}
\label{AssErrors3}
  C_V = \sup_{t \in \Z} \max_{1 \le \ell \le L} \sigma_{t\ell}^2 < \infty. 
\end{equation}
The variances of the idiosyncratic noise  are bounded away from zero,
\begin{equation}
\label{AssErrors4}
 \sigma_{t0}^2 \ge \underline{\sigma}_0^2 > 0,
\end{equation}
whereas for general $ d > 1 $ the latter assumption is replaced by
\begin{equation}
\label{AssErrors4GeneralD}
  \lambda_{\min}( \bfSigma_{\bfvareps,t} ) \ge \underline{\lambda}_{\min} > 0 
\end{equation}
for some constant $ \underline{\lambda}_{\min} $. Further, we assume that all $ \varepsilon_t^{(\ell,\nu)} $'s attain densities $ p_t^{(\ell,\nu)} $ satisfying
\begin{equation}
\label{AssErrors5}
\sup_{t \ge 1} \int_{\R} | p_t^{(\ell,\nu)}(z+x) - p_t^{(\ell,\nu)}(z) | \, dz = O( |x| ).
\end{equation}
This condition is satisfied, if the densities are of bounded variation.

The coefficients of the linear processes are assumed to satisfy the  algebraic weak dependence decay condition
\begin{equation}
\label{AssDecay}
 \max_{1 \le \ell \le L} \| \matC_j^{(\ell)} \|_\infty =   \max_{1 \le \ell \le L}  \max_{1 \le \nu \le d} \| \vecc_{j}^{(\ell,\nu)} \|_\infty \le C \max(j,1)^{-7/2-\theta}
\end{equation}
for some $ \theta > 0 $. Obviously, it then holds $ \max_{1 \le \ell \le L} \| \matC_j^{(\ell)} \bfSigma_{\bfvareps}^{(\ell)}{}^{1/2}  \|_\infty = O( \max(j,1)^{-7/2-\theta} ) $ as well. Therefore, one can absorb the matrices $ \bfSigma_{\bfvareps}^{(\ell)}{}^{1/2}  $ into the coefficient matrices $ \matC_j^{(\ell)} $ and hence assume that the innovations are independent with independent coordinates, since our results only require to impose condition (\ref{AssDecay}) on the coefficient matrices.

Lastly, we impose the following regularity condition on $\vecH$. The coordinate functions $ H_\nu $ of $ \vecH = ( H_1, \ldots, H_d ) $ are Lipschitz continuous with constant $L_H $,
\begin{equation}
\label{AssGLIP}
  \sup_{\vecx \not= \vecy} \frac{ | H_\nu( \vecx ) - H_\nu( \vecy )| }{ \| \vecx - \vecy \|_F } \le L_H
\end{equation}
and there exists a random $ d \times (m+1) $ matrix $ \matxi = ( \bfxi_0, \ldots, \bfxi_m ) $ which is independent of $ \{ \vecU_t \}  $ with $ \sup_{\nu \ge 1} \EE \| H_\nu(\matxi) \|_F^{2+\delta} < \infty $, where $ \delta $ is as in Assumption (\ref{AssErrors2}). This implies
\begin{equation}
\label{MomentsG}
  \sup_{\nu \ge 1} \EE | \epsilon_t^{(\nu)} |^{2+\delta} < \infty.
\end{equation}
In (\ref{AssGLIP})  $ \| \matA \|_F = \left( \sum_{i=1}^{n_1} \sum_{j=1}^{n_2} a_{ij}^2 \right)^{1/2} $ denotes the Frobenius matrix norm of a $ n_1 \times n_2 $ matrix $ \matA = (a_{ij})_{ij} $. The operator matrix norm associated to the Euclidean vector $2$-norm, the spectral norm, is denoted $ \| \matA \|_2 $ . Recall that $ \| \matA \|_2 = \sqrt{ \lambda_{\max}( \matA^\top \matA )}$, where $ \lambda_{\max}( \cdot ) $ denotes the largest eigenvalue, and that it is given by the spectral radius $ \rho(\matA) = \lambda_{\max}( \matA ) $ if $ \matA $ is symmetric. Also recall that both matrix norms are submultiplicative and satisfy
 $ \| \matA \|_2 \le \| \matA \|_F \le \sqrt{r(\matA^\top \matA )} \| \matA \|_2 $, where $ r(\matA) = \min(n_1, n_2) $ denotes the rank of $ \matA $.

\subsection{Assumptions on the regressors}

Throughout the rest of the paper, for simplicity of presentation, we change our notation and denote by $ \PP = \PPP_{|\vecX} $ the conditional probability given $ \vecX_t = \vecx_t $, $t \ge 1 $. Concerning the regressors and their relationship with the errors, we need to impose further regularity assumptions. We assume that the regressors are bounded and satisfy for positive definite $ \bfSigma_\vecx \in \R^{p \times p} $, constants $ 0 \le \eta $, $ C < \infty $ and a compact interval $ A\subset (0,\infty) $  
\begin{equation}
	\label{AssRegr1}
	\left\| \frac{1}{T} \sum_{t=1}^T \vecx_t \vecx_t^\top - \bfSigma_\vecx \right\|_2 \le C  \frac{p^\eta}{\sqrt{T}}\ \ \text{and} \ \ \text{spec}( T^{-1} \matX_T^\top \matX_T ) \subset A.
\end{equation}
For fixed $p$ the convergence is for  $ T \to \infty $ and the condition on the spectra of the Gram matrices $ T^{-1} \matX_T^\top \matX_T $ then follows from $ 0 < \lambda_1( \bfSigma_\vecx ) \le \cdots \le \lambda_p( \bfSigma_\vecx ) < \infty $. Compactness of the spectrum is a common assumption in the literature, see e.g. \cite{LiuYu2013}. If $p$ grows as $ T \to \infty $,  (\ref{AssRegr1}) assumes  compactness of the spectra of $ T^{-1} \matX_T^\top \matX_T $, $ T, p \ge 1 $, and a convergence rate for the associated Gram matrices. Our condition  is mild in view of the known results about convergence rates of covariance matrices. Especially, if the rows of $ \matX_T $ are realizations of independent subgaussian random vectors, then (\ref{AssRegr1}) holds for $ \eta = 1/2 $ with high probability of at least $ 1 - 2 e^{-(p+T^{\varepsilon_1})\varepsilon_2} $, for arbitrarily small $ \varepsilon_1, \varepsilon_2 > 0 $, see \cite[Th.~4.7.1]{Vershynin2018} and the subsequent discussion or \cite{Wainwright2019}. The necessity of the condition $ p = o(T) $ has also been identified and studied, see \cite[Th.~3.1]{LedoitWolf2004}, in the sense of consistency.


The central limit theorem for the least squares estimates $ \bhbeta_T^{(\nu)} $ needs an assumptions on  the matrices
\begin{equation}
	\label{DefGXe}
	\bfGamma_T^{(\nu)} = \frac{1}{T} \sum_{s=1}^T \sum_{t=1}^T \vecx_s \vecx_t^\top  \Cov( \epsilon_s^{(\nu)}, \epsilon_t^{(\nu)} ),
\end{equation}
namely that, a.s.,
\begin{equation} 
	\label{AssGXe}
	0 < \liminf_{T \to \infty} \lambda_{\min}( \bfGamma_T^{(\nu)} ) < \limsup_{T \to \infty} \lambda_{\max}( \bfGamma_T^{(\nu)} ) < \infty,
\end{equation}
where $ \lambda_{\min}( \matA) $ ($ \lambda_{\max}( \matA ) $) denotes the smallest (largest) eigenvalue of a symmetric matrix $ \matA $. Assumption (\ref{AssGXe}) on the asymptotic variance of the scaled average of the $ \vecx_t \epsilon_t^{(\nu)} $ is technical but typical. If, for example, $ \vecX_t $ and $ \epsilon_t^{(\nu)} $ are independent i.i.d. sequences, then $ \bfGamma_T $ simplifies to $ \bfGamma_T^{(\nu)} = (T^{-1} \sum_{t=1}^ T \vecx_t \vecx_t^\top ) \sigma_\nu^2 $. Then  (\ref{AssGXe}) follows from (\ref{AssRegr1}). 



\section{Asymptotic results}
\label{Sec: Theory}

The results of this section deal with the consistency of the least squares, ridge and $ \ell_s $-penalized estimators in terms of element-wise weak and $ L_1 $-consistency as well as bounds for the sample prediction error. For least squares we provide consistency results under weak moment conditions and allowing for unbounded errors. To achieve this, the method of proof avoids Bernstein-type concentration inequalities. A further result considers slowly decaying dependence, expressed by the coefficients of the linear processes, in the sense of the algebraic decay condition (\ref{AssDecay}).  In this case the number of hidden units, $p$, may grow only very slowly. When considering the common assumption of geometric weak dependence, the result in terms of $p$ can be strengthened to $ p\log(p) = o(T) $ up to a logarithmic factor, whereas our new result requires $ p = o(T) $ under a mild and natural condition on the design matrix. We also establish asymptotic normality for the non-stationary setting under study.

For ridge estimation we show that under weak general assumptions, namely mean zero errors with variances bounded by the sample size and a regularization parameter $ \lambda_T = o(T) $, the ridge estimator is asymptotically equivalent to the least squares estimate up to a $ o_\PP( 1/ \sqrt{T} ) $ term. Consequently, the results for least squares estimation essentially carry over. A further result allows for a random regularization parameter, thus being more relevant in view of applications where $ \lambda_T $ is often selected in a data-adaptive way. The CLT also carries over, essentially as long as the mild condition $ \EE( \lambda_T ) = o(1) $ holds. 

Lastly, we study $ \ell_s $-penalized estimation extending and complementing known results for i.i.d. and stationary errors.

Recall that, for brevity, $ \PP = \PPP_{Y|X} $. Many of the results of this section, however, immediately carry  over to statements w.r.t. $ \PPP = \PPP_X \PPP_{Y|X} $ or w.r.t. $ \PPP_{\bfTheta|X} \PPP_{Y|X} $. The later is of interest for an artificial neural network with random weights $ \bfTheta = (\matW^{(1)}, \dots, \matW^{(r)}, \vecb_1, \dots, \vecb_r ) $ where we have
$  \PPP_X = \PPP_Z \otimes \PPP_{\bfTheta} $. 

\subsection{Least squares}
Let us first study the consistency of the least squares solution $ \bhbeta_T = ( \wh{\bfbeta}_T^{(\nu)} )_\nu $, i.e.  consistency of the tuned weights connecting the hidden layer and the $ \nu$th output of the neural network. The first result provided below assumes $ p^{2\eta} < T $ and $ p^{2\eta} = o(T) $ as common in asymptotic studies. These conditions typically guarantee the existence of the inverse of $ T^{-1}\matX_T^\top \matX_T $. Especially, $ p = o(T) $ is known to be a necessary condition for the convergence to $ \bfSigma_\vecx $ even for i.i.d. samples, see \cite[Th.~3.1]{LedoitWolf2004}. That condition is also not not restrictive for ELMs, as deep learners are typically applied to large data sets. It is worth recalling at this point that $p$ is number of hidden units of the output layer, and the results below allow for multilayer networks covering deep learners as well as fully connected networks: The number of parameters is essentially not constrained. 

The first result establishes consistency under the following {\em growth of sum moment} assumption describing the growth of the absolute moments of $ \sum_{t=1}^N \vecx_t \epsilon_t^{(\nu)} $, $ \nu \ge 1$, as a function of $N$: There exists a constant $ C_\nu < \infty $ and $ \gamma' \ge  2 $ such that for all $ 2 \le \gamma \le \gamma' $ and all $ N \in \N $
\begin{equation}
\label{AssOrderOfMoments}
\sup_{1 \le j} \EE \left| \sum_{t=1}^N x_{tj} \epsilon_t^{(\nu)} \right|^\gamma \le C_\nu  N^{\gamma/2}.
\end{equation}
Assumption (\ref{AssOrderOfMoments}) is rather mild and holds for many weakly dependent series. Especially, it holds for the  model in Section~\ref{ModelErrors} under a slightly strengthened condition on the algebraic weak dependence condition on the coefficients, see Theorem~\ref{ThCONS}~(ii) and Section~\ref{Sec: Proofs} for details. 

The result under (\ref{AssOrderOfMoments}) for least squares supervised learning by regression and (multilayer) ELMs is as follows:

\begin{theorem} (Consistency Under Limited Growth of Sum Moment)\\
	\label{ThCONS} 
	Suppose that Assumption (\ref{AssRegr1})
	\begin{itemize}
	\item[(i)] If (\ref{AssOrderOfMoments}) holds and the sample size $ T $ ensures
	\[
	   \left\{ \begin{array}{ll} 
	  			T \ge p, p = o(T) & \text{if } 0 \le \eta \le 1/2, \\ 
	  			p^{2\eta} = o(T), & \text{if } \eta > 1/2,
	  					\end{array} \right.
	\]
	then for each $ \delta > 0 $
	\[
	  \PP( \| \bhbeta_{T\nu} - \bfbeta_\nu \|_2 > \delta ) \le K C_\nu \left( \sqrt{ \frac{p}{T} } \right)^{\gamma} \delta^{-\gamma},
	\] 
	for some constant $K > 0 $,	especially implying
	\[
	  \bhbeta_{T\nu} = \bfbeta_\nu + O_\PP\left( \sqrt{ \frac{p}{T} } \right),
	\]
	and
	\begin{equation*}
		\label{L2ConvCONS}
		\EE \left\| \bhbeta_{T\nu} - \bfbeta_\nu \right\|_2^2 \le \left( C \frac{p^\eta}{\sqrt{T}} + \lambda_{\min}  \right) C_\nu \left( \frac{p}{T} \right).
	\end{equation*}
  for a constant $ C $, where $ \lambda_{\min} = \inf_{T,p} \text{spec}( T^{-1} \matX_T^\top \matX_T ) $ and the upper bound is $ o(1) $.
	\item[(ii)]  
	If Assumptions (\ref{AssErrors1})-(\ref{AssGLIP}) on the error process hold with  $ \delta = 2 $ in (\ref{AssErrors1}), i.e., the innovations of the linear process have a finite fourth moment, and (\ref{AssDecay}) holds with $ \theta = 3/2 $, i.e., the coefficients $ c_j^{(\ell,\nu)} $ decay algebraically as $ j^{-5-\delta'} $ for some $ \delta'>0$, then (\ref{AssOrderOfMoments}) is fulfilled with $ \gamma' = 2 $, such that the assertions of (i) hold.
	\item[(iii)] Let an ELM be specified by a single hidden layer neural net (\ref{ELM1}) respectively a mulilayer deep learning network (\ref{ELM2}) with errors satisfying Assumptions (\ref{AssErrors1})-(\ref{AssGLIP}). If (\ref{AssErrors1}) holds with $ \delta = 2 $ and if (\ref{AssDecay}) holds with $ \theta = 3/2 $, i.e.,  $ c_j^{(\ell,\nu)} = O(  j^{-5-\delta'} ) $ for some $ \delta'>0$, then (\ref{AssOrderOfMoments}) holds with $ \gamma' = 2 $, such that the assertions of (i) hold.
	\end{itemize}
\end{theorem}


The above results imply bounds and thus consistency statements for the full estimate $ \wh{\matB}_T $ as well. But here the choice of a norm matters, of course. For example, the following corollary follows easily from the union bound.

\begin{corollary} Under the conditions of Theorem~\ref{ThCONS} (i) 
	\[
	  \PP( \| \wh{\matB}_T - \matB \|_F > \delta )  \le \frac{K^2}{\delta^2} \left( \sum_{\nu=1}^d C_\nu \right)  \frac {p}{T}
	\]
	for some constant $K$.
\end{corollary}

The  results of  Theorem~\ref{ThCONS}  can also be used to bound the empirical mean-square prediction error
\[
 \widehat{MSPE}_T = \frac{1}{T} \sum_{t=1}^T | \vecx_t^\top \bhbeta_{T\nu} - \vecx_t^\top \bfbeta_{\nu} |^2
\]
used to estimate the mean-square prediction error
\[
  MSPE_T = \frac{1}{T} \sum_{t=1}^T \EE( \wh{Y}_{T}^{(\nu)} - Y_T^{(\nu)} )^2,
\]
where $ \wh{Y}_{T}^{(\nu)} = \vecx_t^\top \bhbeta_{T\nu} $.

\begin{corollary} Under the assumptions of Theorem~\ref{ThCONS} it holds
	\[
	\widehat{MSPE}_T = O_\PP\left( \frac{p}{T}  \right)
	\]
\end{corollary}

The above estimate immediately follows by noting that  $  \widehat{MSPE}_T  = \frac{1}{T} \| \matX_T( \bhbeta_{T\nu} - \bfbeta_\nu) \|_2^2 $, so that
\[
  \widehat{MSPE}_T \le \lambda_{\max}( T^{-1} \matX_T^\top \matX_T ) \| \bhbeta_{T\nu} - \bfbeta_\nu \|_2^2 = O_\PP\left( \frac{p}{T}  \right).
\]
\begin{remark}
	Our convergence rates for least squares estimation under dependence are in qualitative agreement with related results for least squares function approximation for linear vector spaces under i.i.d. sampling, see e.g. \cite{GyoerfiKohlerEtAl2002}. 
\end{remark}

Theorem~\ref{ThCONS} uses (\ref{AssOrderOfMoments}) to bound the estimation error instead of following the popular approach to rely on an exponential (Hoeffding- or Bernstein-type) concentration inequality, which requires stronger assumptions on the degree of dependence for nice results. The following theorem works under the algebraic weak dependence condition on the coefficients. It allows for increasing dimension $p$, but only at a very slow unsatisfactory rate.  

\begin{theorem} (Consistency Under Algebraic Weak Dependence)\\
\label{THCONS2}
		Suppose that Assumptions (\ref{AssErrors1})-(\ref{AssGLIP}) and (\ref{AssRegr1}) are fulfilled and $ |x_{tj} \epsilon_t^{(\nu)} | \le 1 $ for all $ t $, $ j = 1, \ldots, p $ and $ \nu \ge 1 $. Then
		\[
				\bhbeta_{T\nu} \to \bfbeta_\nu,
		\]
		as $  T, p \to \infty $, in probability, if 
		\[
		  \log(p) \sqrt{p} = o\left( \log^{1-a}( T )  \right),
		\]
		for $ 0 < a < 1 $.
\end{theorem}

For geometrically decaying weak dependence the situation is much better and the above result can be improved in terms of the convergence rate.

\begin{theorem}  (Consistency Under Geometric Weak Dependence)\\
	\label{THCONS3}
	Suppose that Assumptions (\ref{AssErrors1})-(\ref{AssErrors5}), (\ref{AssGLIP}) and  (\ref{AssRegr1}) are fulfilled and the coefficients decay geometrically, i.e. $ c_j^{(\ell,\nu)} = O( \rho^j ) $ for some $ 0 \le \rho < 1 $, and $ |x_{tj} \epsilon_t^{(\nu)} | \le 1 $ for all $ t $, $ j = 1, \ldots, p $ and $ \nu \ge 1 $. Then
	\[
	\bhbeta_{T\nu} \to \bfbeta_\nu,
	\]
	as $  T, p \to \infty $, in probability, if $ p \log (p) / T  = o(1) $ and $ \log^2(T) / \sqrt{p} = o(1) $.
\end{theorem}

Let us now study the asymptotic distribution theory. 
In our general setting we need the notion of asymptotic normality. Recall that
a sequence, $ \vecZ_n $, $n \ge 1$, of random vectors is called asymptotically normal with centering vectors $ \bfmu_n $ and asymptotic covariance matrices $ \bfSigma_n $, denoted by $ \calA\calN( \bfmu_n, \bfSigma_n ) $, if for every vector $ \veca $ with $ \veca^\top \bfSigma_n \veca > 0 $ for sufficiently large $n$ the univariate sequence $ \veca^\top \vecZ_n $, $n \ge 1 $, is asymptotically normal. If $ \bfmu_n \to \bfmu $ and $ \bfSigma_n \to \bfSigma $, as $ n \to \infty $, then the multivariate central limit theorem (CLT) $ Z_n \stackrel{d}{\to} \calN( \bfmu, \bfSigma) ) $, $ n \to \infty $, follows.

The following theorem provides asymptotic normality under weak algebraic dependence. If, additionally, $ \matGamma_T^{(\nu)} $ converges w.r.t. $ \| \cdot \|_2 $ or $ \| \cdot \|_F $, it implies the central limit theorem.

\begin{theorem} 
	\label{ThASNOR}
	Suppose that Assumptions (\ref{AssErrors1})-(\ref{AssGLIP}) are satisfied.Then
	the following assertions hold.
	\begin{itemize}
		\item[(i)] $ (T v_{\nu}^2)^{-1/2} \sum_{t=1}^T \epsilon_t^{(\nu)} \stackrel{d}{\to} \calN(0, 1), $ as $ T \to \infty $, where $ v_\nu^2  = \lim_{T \to \infty} v_{T\nu}^2 > 0 $ with $ v_{T\nu}^2 = \Var( T^{-1/2} \sum_{t=1}^T \epsilon_t^{(\nu)}  ) $
		needs to be assumed if $ \vecH \not= \id $.
		\item[(ii)] $ \vecW_T = T^{-1/2} \sum_{t=1}^T \vecx_t  \epsilon_t^{(\nu)} \sim \calA\calN(0, \bfGamma_T), $
		as $ T \to \infty $.
		\item[(iii)] Under Assumption (\ref{AssGXe}), the least squares estimator $ \bhbeta_{T\nu} $ of $ \bfbeta_\nu $ is asymptotically normal,
		\[
		\sqrt{T}( \bhbeta_{T\nu} - \bfbeta_\nu ) \sim \calA\calN( \vecnull, \bfSigma_\vecx^{-1} \bfGamma_T^{(\nu)} \bfSigma_\vecx^{-1} ),
		\]
		for $ \nu = 1, \ldots, d $.
	\end{itemize}
\end{theorem}

\subsection{Ridge regression}

The following first result provides the consistency of the ridge estimator, provided the 2-norm of true (conditional) mean $ \matX_T \bfbeta_\nu $ (the true net output) is $ O(T) $, if $ T, p \to \infty $. This holds true, if, for example, the true parameter $ \bfbeta_\nu $ is $ \ell_0 $-sparse, i.e. if $ \| \bfbeta_\nu \|_{\ell_0} = s $ for  $s \in \N$. The result asserts that under surprisingly weak conditions on the errors the ridge estimator is asymptotically equivalent to the least squares estimator with convergence rate $ 1/\sqrt{T} $. Consequently, any additional set of sufficient conditions for consistency of least squares yields consistency of the ridge estimator.  In this section we assume $ \lambda_T > 0 $ for all $ T $.

\begin{theorem}
\label{THRIDGE2}
Suppose that the spectrum of $T^{-1} \matX_T^\top \matX_T $ is covered by a compact interval not containing $0$ for large enough $T$, 
\begin{equation}
\label{CondDesignBeta}
  \| \matX_T \bfbeta_\nu \|_2 = O(T)
\end{equation}
if $ T, p \to \infty $, and $ \epsilon_t^{(\nu)} $, $ t \ge 1 $, are mean zero error terms with $  \Var( \epsilon_t^{(\nu)}  ) = O(T) $. If the regularization parameter $ \lambda_T $ is non-random and satisfies
\[
  \lambda_T = o( T ),
\]
or is random and satisfies
\[
  \EE( \lambda_T^2 ) = o(1),
\]
then the ridge estimator is asymptotically equivalent to the least squares solution,
\[
  \bfbeta_{T\nu}^{(R)} = \wh{\bfbeta}_\nu + o_\PP(1/\sqrt{T}),
\]
as $ T, p \to \infty $. The assertions of Theorems~\ref{ThCONS}-\ref{THCONS3}  carry over to the ridge estimator under the assumptions stated there, especially then it holds
$ \bfbeta_{T\nu}^{(R)}  = \bfbeta_\nu + o_\PP(1)$, as $ T, p \to \infty $ with 
\begin{itemize}
\item $ p/T = o(1) $ under Theorem~\ref{ThCONS} (i), if $ 0 \le \eta \le 1/2 $, resp. $ p^{2\eta} = o(T) $, if $ 1/2 < \eta $,  (limited growth of sum moments), 
\item $ \log(p) \sqrt{p} = o( \log^{1-a}(T) ) $ under Theorem~\ref{THCONS2} (algebraic weak dependence) resp. 
\item $ p \log(p) / T = o(1) $ and $ \log^2(T) / \sqrt{p} = o(1) $ under Theorem~\ref{THCONS3} (geometric weak dependence). 
\end{itemize}
\end{theorem}

The next result replaces the condition (\ref{CondDesignBeta}) by a sparseness condition on the true asymptotic solution $ \bfbeta_\nu $ and the asymptotic covariance matrix $ \bfSigma_\vecx $, which holds for a $\ell_2 $-representable learning problem. It also allows for the case that $ \lambda_T/T $ converges to some non-zero limit.

\begin{theorem}  	
\label{THRIDGE}
	Suppose that the  assumption
\begin{equation}
\label{CondSparsity}
  \| \bfSigma_\vecx \bfbeta_\nu \|_2 = O(1)
\end{equation}
holds. 
\begin{itemize}	
\item[(i)]  If $ \lambda = \lambda_T $ is (potentially randomly) selected such that
\[
  \frac{\lambda_T}{T} = o_\PP( 1/ \sqrt{p} ), 
\]
then the assertions of Theorems~\ref{ThCONS}-\ref{THCONS3} hold true for the ridge estimator $ \bhbeta_{T\nu}^{(R)} $, under the assumptions stated in these theorems. Especially, 
\[
  \bhbeta_{T\nu}^{(R)} = \bfbeta_\nu + o_\PP(1).
\]
\item[(ii)] If $ \frac{\lambda_T}{T} \to \lambda^0 $ for some constant $ \lambda^0 $, then
\[
  \wh{\bfbeta}_{T\nu} = (\bfSigma_\vecx + \lambda^0 \matid )^{-1} \bfSigma_\vecx \bfbeta_\nu + o_P(1).
\]
\end{itemize}
\end{theorem}

\begin{remark} In Theorem~\ref{THRIDGE} the dimension $p$ of the regressors  may increase with the sample size $T$. Hence, Assumption (\ref{AssRegr1}) implicitly assumes a convergence rate, so that the $p \times p $ dimensional matrix $ T^{-1} \matX_T^\top \matX_T $ still converges in the Frobenius matrix norm. 
\end{remark}

\begin{remark}
At first glance condition (\ref{CondSparsity}) may look restrictive, but it is a natural one for representable signals as already discussed in Section~\ref{Sec: Regression ML}. 
\end{remark}

The consistency of the ridge estimator for linear regression models has been studied in the statistical literature to some extent. \cite{Silva2014} shows its consistency for i.i.d. errors when using the plug-in estimator for $\lambda^* $, i.e. $ \wh{\lambda} = p \wh{\sigma}_T^2 / \| \wh{\bfbeta}_T \|^2_2 $ where $ \wh{\sigma}_T^2 $ is the usual estimator for $ \sigma^2 $ and $ \wh{\bfbeta}_T $ the least squares estimator, under the assumption that $ \lambda_{\max}[ ( T^{-1} \matX_T^\top \matX_T )^{-1} ] = 1 / \lambda_{\min}( T^{-1}  \matX_T^\top \matX_T )  = o(1) $, which is, however, not reasonable for regressors satisfying the natural condition $ T^{-1} \matX_T^\top \matX_T \to \bfSigma_\vecx $ imposed in (\ref{AssRegr1}).  Contrary, our results assume (\ref{AssRegr1}) and allow for dependent non-stationary errors as well as increasing dimension $p$. Further, they cover choices $ \lambda_T = o(T) $  of the regularization parameter ensuring improved estimation in finite samples and consistency as $ T \to \infty $ and the validity of a central limit theorem, see below. Consistency results assuming sparsity, condition (\ref{AssRegr1}), Gaussian i.i.d. error terms and non-random sequences of regularization parameters can also be found in \cite{LiuYu2013}.

Lastly, the following results provide asymptotic normality.

\begin{theorem} 
	\label{THASNORMRIDGE} 
	Suppose that the conditions of Theorem~\ref{ThASNOR} are fulfilled and assume that the variances of the idiosyncratic errors are bounded.
	\begin{itemize}
		\item[(i)] If the regularization parameter $ \lambda = \lambda_T \ge 0 $ is non-random and satisfies
		\[
		\lambda_T = o(1),
		\]
		then the assertions of Theorem~\ref{ThASNOR}  carry over to the estimator $ \wh{\bfbeta}_T^{(R)} $. Especially,
		\[
		\sqrt{T}( \bhbeta_{T\nu}^{(R)} - \bfbeta_\nu ) \sim \calA\calN( \vecnull, \bfSigma_\vecx^{-1} \bfGamma_T^{(\nu)} \bfSigma_\vecx^{-1} ).
		\]
		for $ \nu = 1, \ldots, d $. 
		\item[(ii)] If the regularization parameter $ \lambda = \lambda_T \ge 0 $ is random and converges to $0$ in mean-square, i.e.
		\begin{equation}
		\label{MeanSqConvLambda}
		\EE( \lambda_T^2 ) = o(1),
		\end{equation}
		as $ T \to \infty $, then 
		\[
		\sqrt{T}( \bhbeta_{T\nu}^{(R)} - \bfbeta_\nu ) \sim \calA\calN( \vecnull, \bfSigma_\vecx^{-1} \bfGamma_T^{(\nu)} \bfSigma_\vecx^{-1} ).
		\]
		for $ \nu = 1, \ldots, d $. 
	\end{itemize}
\end{theorem}

The assertion and the condition on $ \lambda_T $ in Theorem~\ref{THASNORMRIDGE} (i) are in agreement with the results obtained in \cite{LiuYu2013}, which consider Gaussian i.i.d. errors and non-random choices of $ \lambda_T $. Assertion (ii) provides a convenient criterion when the regularization parameter is estimated from data. Especially, since
\[
\| \lambda_T \|_{L_2} \le  \| \lambda_T  -  \EE(\lambda_T) \|_{L_2} + \EE( \lambda_T ),
\] 
condition (\ref{MeanSqConvLambda}) is fulfilled and  the ridge estimator is asymptotically normal, if $ \lambda_T $ converges to its expectation in mean-square and $ \EE( \lambda_T ) = o(1) $. 

\subsection{$\ell_s$-penalized estimation (LASSO)}

Let us now study the $ \ell_s $-norm penalized least squares estimator by providing bounds for the sample MSPE. Our results improve upon \cite{BuehlmannGeer2011}, where such results are shown for the lasso estimator and Gaussian resp. i.i.d. errors. For independent errors it is known that $ \lambda $ needs to be selected proportional to $ \sqrt{ \log(p) / T } $. For dependent errors the rate is worse and needs to be proportional to $ \sqrt{ p\log(p) / T } $. 

\begin{theorem} 
	\label{THCONSPRED} 
	Suppose that Assumptions (\ref{AssErrors1})-(\ref{AssErrors5}) and (\ref{AssGLIP}) are fulfilled and the coefficients decay geometrically, i.e. $ c_j^{(\ell,\nu)} = O( \rho^j ) $ for some $ 0 \le \rho < 1 $, and $ |x_{tj} \epsilon_t^{(\nu)} | \le 1 $ for all $ t $, $ j = 1, \ldots, p $ and $ \nu \ge 1 $. If the regularization parameter is chosen as
	\begin{equation}
		\label{ChoiceLambda}
	\lambda = C^{-1/2} \sqrt{ \log\left( \frac{2p}{ \alpha } \right)  \left( \frac{p}{T} + \frac{ \sqrt{p} \log^2(T) }{ T} \right) },
	\end{equation}
	for some constant $ C $ depending on $ \rho $, 
	then for large enough $T $ and with probability at least $ 1 - \alpha $, $ \alpha \in (0,1) $, the sample prediction error associated to the $ \ell_s $-penalized least squares estimator, $ 1 \le s \le 2 $, can be bounded by
	\begin{equation}
	\label{MSPE_Bound}
	\wh{MSPE}_T^{(\ell_s)} = \frac{1}{T} \sum_{t=1}^T ( \vecx_t^\top \bhbeta_{T\nu}^{(\ell_s)} - \vecx_t^\top \bfbeta_\nu )^2 \le 2 C^{-1/2} \sqrt{ \log\left( \frac{2p}{ \alpha } \right)  \left( \frac{p}{T} + \frac{ \sqrt{p} \log^2(T) }{ T} \right) } \| \bfbeta_\nu \|_s.
	\end{equation}
	Therefore, MSPE-consistency follows if $ p = o(T) $ and $ \| \bfbeta_\nu \|_s = O(1) $, or if $ \bfbeta_\nu $ satisfies the sparsity assumption
	\[
		\| \bfbeta_\nu \|_s = o\left( \left[ \log \left( \frac{2p}{ \alpha } \right)  \left( \frac{p}{T} + \frac{ \sqrt{p} \log^2(T) }{ T} \right) \right]^{-1/2} \right),
	\]
	as $ T, p \to \infty $.
\end{theorem}

We obtain the following corollary.

\begin{corollary} Suppose the conditions of Theorem~\ref{THCONSPRED} are satisfied, the learning problem is $ \ell_s $-representable for some $ 1 \le s \le 2 $ and satisfies
	\[
		\| \bfSigma_\vecx^{-1} \|_{s,op} = O(1)
	\] 
and 
\[
  \| \bftheta_{\nu,1:p} \|_s = o\left( \left[ \log \left( \frac{2p}{ \alpha } \right)  \left( \frac{p}{T} + \frac{ \sqrt{p} \log^2(T) }{ T} \right) \right]^{-1/2} \right).
\]
Then the bound (\ref{MSPE_Bound}) holds and $ \wh{MSPE}_T^{(\ell_s)}  = o_P(1) $.
\end{corollary}

The above result requires boundedness of the inverse $ \bfSigma_\vecx $ with respect to the operator norm $ \| \matA \|_{s,op} $ associated to the $ \ell_s $-norm. 

The above results require geometrically decaying coefficients of the underlying linear processes. By using approach and results of Theorem~\ref{ThCONS} that condition can be weakened to a algebraic decay yielding the following theorem.

\begin{theorem}
	\label{THCONSPRED2}
	Under the assumptions of Theorem~\ref{ThCONS}~(ii) the choice
	\[
	  \lambda = \left( \frac{C_\nu}{1-\alpha} \right)^{1/\gamma} \sqrt{ \frac{p}{T} } 
	\]
	for some $ \alpha \in (0,1) $ ensures that with probability at least $ 1 - \alpha $ the
	$ \ell_s $-penalized least squares estimator, $ 1 \le s \le 2 $, can be bounded by
	\[
	  \wh{MSPE}_T^{(\ell_s)} \le  2  \left( \frac{C_\nu}{1-\alpha} \right)^{1/\gamma} \sqrt{ \frac{p}{T} } \| \bfbeta_\nu \|_s 
	\] 
	and $ MSPE $-consistency follows, if $ p = o(T) $ and $ \| \bfbeta_\nu \|_s = O(1) $ or for unconstrained $ p, T \to \infty $
	\[
	  \| \bfbeta_\nu \|_s  = o\left( \left[ \left( \frac{C_\nu}{1-\alpha} \right)^{1/\gamma} \sqrt{ \frac{p}{T} } \right]^{-1}   \right).
	\]
\end{theorem}

\section{Proofs}
\label{Sec: Proofs}

Let us first show (\ref{MomentsG}).  Using the $ C_r $-inequality we get for $ 1 \le \ell \le 2 + \delta $
\begin{align*}
  \EE | \epsilon_t^{(\nu)} |^\ell & =
    \EE | [H_\nu( \vecU_{t:m} ) - H_\nu( \bfxi )]  - [ \EE H_\nu( \vecU_{t:m} ) - H_\nu( \bfxi ) ] |^\ell \\
    & \le 2^{\ell-1} \EE  | H_\nu( \vecU_{t:m} ) - H_\nu( \bfxi ) |^\ell  
    	+ 2^{\ell-1} \EE |  H_\nu( \bfxi ) - \EE_{\vecU_{t:m}} H_\nu( \vecU_{t:m} ) |^\ell, 
\end{align*}
where by  Assumption (\ref{AssGLIP}) $ \EE  | H_\nu( \vecU_{t:m} ) - H_\nu( \bfxi ) |^\ell \le L_H \EE \| \bfxi - \vecU_{t:m} \|_F^\ell  $  and, by independence of $ \bfxi $ and $ \vecU_{t:m} $,
\begin{align*}
  \EE |  H_\nu( \bfxi ) - \EE_{\vecU_{t:m}} H_\nu( \vecU_{t:m} ) |^\ell 
  & = \EE_{\bfxi} \EE_{\vecU_{t:m}}  |  H_\nu( \bfxi ) - H_\nu( \vecU_{t:m} ) |^\ell\\ 
  & \le L_H \EE \| \bfxi - \vecU_{t:m} \|_F^\ell.
\end{align*}
Further, $ \EE \| \bfxi - \vecU_{t:m} \|_F^\ell = O(  \EE | \| \bfxi \|_F^\ell + \EE \| \vecU_{t:m} \|_F^\ell ) $, where by Jensen's inequality
\begin{align*}
 \EE \| \vecU_{t:m} \|_F^\ell & = \EE\left(  \sum_{j=0}^m \sum_{\nu=1}^d ( U_{t-j}^{(\nu)} )^2   \right)^{\ell/2} 
 \le ((m+1)d)^{\ell/2-1} \sum_{j=0}^m \sum_{\nu=1}^d \EE  \left| U_{t-j}^{(\nu)} \right|^\ell < \infty,
\end{align*}
observing that $ \EE  \left| Z_{t-j}^{(\nu)} \right|^\ell < \infty $ follows from (\ref{AssErrors2}) since $ d $ and $L$ are fixed in our treatment.

Recall that $ \xi_n = O_\PP( r_n ) $, if $ \PP( |\xi_n | / r_n > \delta )  $ is arbitrarily small for all  $n$ if $ \delta $ is large enough.

\begin{lemma} 
\label{AuxLemma}
	If $ \PP( | \xi_n | > \delta ) \le C r_n / \delta^\gamma $ for some $ \gamma > 0 $, for all $n$ and $ \delta > 0 $ and some constant $C$, then $ \xi_n = O_{\PP}( r_n^{1/\gamma} ) $.
\end{lemma}

	
\begin{lemma}
\label{ConvRateInverse}
Let $ \matA $ and $ \matB $ be square matrices with $ \lambda_{\min}(\matA) > 0$ and $ \lambda_{\min}( \matB ) > 0 $. Then for any submultiplicative matrix norm $ \| \cdot \| $ which is dominated by the spectral matrix norm $ \| \cdot \|_2 $ 
\[
  \| \matA^{-1} - \matB^{-1} \| \le \frac{\| \matA - \matB \|}{\lambda_{\min}(\matA)  \lambda_{\min}(\matB) }. 
\]
\end{lemma}

\subsection{Proof of Theorem~\ref{ThCONS}}

Let us consider the representation
\begin{equation}
\label{repr_diff_bhat_beta}
  \wh{\bfbeta}_{T\nu} = \bfbeta_\nu + \left( T^{-1} \matX_T^\top \matX_T  \right)^{-1}  \frac{1}{T} \sum_{t=1}^T \vecx_t \epsilon_t^{(\nu)}.
\end{equation}
By Jensen's inequality it holds for $ \gamma \ge 2 $
\begin{align*}
  \EE \left\| \frac{1}{T} \sum_{t=1}^T \vecx_t \epsilon_t^{(\nu)} \right\|_2^\gamma
   & = T^{-\gamma} p^{\gamma/2} \EE\left( \frac{1}{p} \sum_{j=1}^p \left| \sum_{t=1}^T x_{tj} \epsilon_t^{(\nu)} \right|^2  \right)^{\gamma/2} \\
   & \le T^{-\gamma} p^{\gamma/2-1} \sum_{j=1}^p \EE \left| \sum_{t=1}^T x_{tj} \epsilon_t^{(\nu)} \right|^\gamma,
\end{align*}
such that Assumption (\ref{AssOrderOfMoments}) yields
\begin{equation}
\label{MomentNu}
  \EE \left\| \frac{1}{T} \sum_{t=1}^T \vecx_t \epsilon_t^{(\nu)} \right\|_2^\gamma
  \le C_\nu T^{-\gamma/2} p^{\gamma/2}. 
\end{equation}
for some constant $C$. Consequently, for any $ \delta > 0 $ we have
\begin{align}
	\label{CrossMeanLittleO1}
  \PP\left( \left\| \frac{1}{T} \sum_{t=1}^T \vecx_t \epsilon_t^{(\nu)} \right\|_2 > \delta \right) & 
  \le C_\nu  \left( \frac{p}{T} \right)^{\gamma/2} \delta^{-\gamma} 
\end{align}
Observe that with $ a = \inf_{p, T \ge 1} \text{spec}( T^{-1} \matX_T^\top \matX_T ) $ in view of (\ref{repr_diff_bhat_beta})
\begin{align}
	\| \bhbeta_T - \bfbeta \|_2 & \le \| ( T^{-1} \matX_T^\top \matX_T )^{-1} \|_2 
	\left\| T^{-1} \sum_{t=1}^T \vecx_t \epsilon_t^{(\nu)} \right\|_2 \nonumber \\ 
	& \le \left( \| (T^{-1} \matX_T^\top \matX_T )^{-1} - \bfSigma_\vecx \|_2 + \| \bfSigma_\vecx \|_2 \right) \left\| T^{-1} \sum_{t=1}^T \vecx_t \epsilon_t^{(\nu)} \right\|_2 \nonumber \\ \label{EstimateBetaHat-Beta}
	& \le\left( C \frac{p^\eta}{T} + a  \right) \left\| T^{-1} \sum_{t=1}^T \vecx_t \epsilon_t^{(\nu)} \right\|_2,
\end{align}
for some constant $ C < \infty $, where in the last step Assumption~\ref{AssRegr1} is used.
By Markov's inequality we can now estimate
\begin{align*}
 \PP( \| \bhbeta_{T\nu} - \bfbeta_\nu \|_2 > \delta )
    &\le \PP\left(  \left( C \frac{p^\eta}{T} + a  \right) \left\| \frac{1}{T} \sum_{t=1}^T \vecx_t \epsilon_t^{(\nu)} \right\|_2 > \delta \right) \\
    & \le \frac{\EE \left\| \frac{1}{T} \sum_{t=1}^T \vecx_t \epsilon_t^{(\nu)} \right\|_2^\gamma}{\delta^\gamma} \left( C \frac{p^\eta}{T} + a  \right)^\gamma \\
    & \le \frac{C_\nu}{\delta^\gamma} \left( \frac{p}{T} \right)^{\gamma/2} \left( C \frac{p^\eta}{T} + a  \right)^\gamma.
\end{align*}
Case $ 0 \le \eta \le 1/2 $:  In this case $ T \ge p $ and $ p/T = o(1) $. Then $ \left( C \frac{p^\eta}{\sqrt{T}} + a  \right)^\gamma $ is bounded by $ K = ( C + a)^\gamma $, since
\[
  \frac{p^\eta}{\sqrt{T}} = \left( \frac{p}{T}\right)^\eta T^{\eta-1/2} \le 1,
\]
such that 
\[
  \PP( \| \bhbeta_{T\nu} - \bfbeta_\nu \|_2 > \delta ) \le K \frac{C_\nu}{\delta^\gamma} \left( \frac{p}{T} \right)^{\gamma/2} \to 0, \qquad p, T \to \infty,
\]
follows. 

Case $ 1/2 < \eta $: Then, by assumption, $ p^{2\eta} = o(T) $, which implies $ p = o(T) $ and  $ \left( C \frac{p^\eta}{\sqrt{T}}  + a  \right)^\gamma  \to a^\gamma $. Therefore,
\[
\PP( \| \bhbeta_{T\nu} - \bfbeta_\nu \|_2 > \delta ) \le \frac{C_\nu}{\delta^\gamma} \left( \frac{p}{T} \right)^{\gamma/2} \left( C \frac{p^\eta}{\sqrt{T}} + a  \right)^\gamma   \to 0, \qquad p, T \to \infty.
\]

Since in both cases  $ \PP( \| \bhbeta_{T\nu} - \bfbeta_\nu \|_2 > \delta )  = O( (p/T)^{\gamma/2} ) $, Lemma~\ref{AuxLemma} implies
\[
  \| \bhbeta_{T\nu} - \bfbeta_\nu \|_2 = O_{\PP}\left(  \sqrt{\frac{p}{T}}  \right).
\]
Lastly, the estimates used to obtain (\ref{EstimateBetaHat-Beta}) also yield
\begin{align*}
  \EE \left\| \bhbeta_{T\nu} - \bfbeta_\nu \right\|_2^2 & \le
    \left( C \frac{p^\eta}{\sqrt{T}} + a  \right)\EE \left\| \frac{1}{T} \sum_{t=1}^T \vecx_t \epsilon_t^{(\nu)} \right\|_2^2 \\
    & \le \left( C \frac{p^\eta}{\sqrt{T}} + a  \right) C_\nu \left( \frac{p}{T} \right)
\end{align*}
where the upper bound is $ o(1) $ as $ p, T \to \infty $, if $ T \ge  p $ and $ p  = o(T) $. This completes the proof of (i). 

Next, let us show assertion (ii). We use the following result which is a special case of \cite[Th.~2, p.~26]{Doukhan1994}.

\begin{lemma} Let $ \{ \xi_t \} $ be a sequence of possibly non-stationary zero mean random variables  with
	$ \EE | \xi_t |^{2+\varepsilon} < \infty $ for all $ t $, $ \varepsilon > 0 $, and $ \alpha $-mixing coefficients $ \alpha_\xi(r) $, $ r \in \N $. If
	\[
	  \sum_{r=1}^\infty  [ \alpha_\xi(r) ]^{ \frac{\varepsilon}{2+\varepsilon} } < \infty,
	\]
	then 
	\[
	  \EE \left|  \sum_{t \in \mathcal{I}} \xi_t \right|^2 = O( | \mathcal{I} | ),
	\]
	for any finite index set $ \calI $. 
\end{lemma}

As shown below in detail, see (\ref{MixingCoefficientsEstimate2}), the mixing coefficients $ \alpha(r) $ of the process $ \xi_t = \vecx_t \epsilon_t^{(\nu)} $ satisfy
\[
  \alpha(r) = O( r^{-1-\frac{2\theta}{3} } )
\]
under Assumption (\ref{AssDecay}).
Consequently, $ \alpha(r)^{ \frac{\varepsilon}{2+\varepsilon}} $ is summable, if
$ (1+2\theta/3)  \frac{\varepsilon}{2+\varepsilon} > 1 $. Select $ \varepsilon = 2 $, corresponding to the fourth order moment condition, we obtain the condition $ \theta > 3/2 $
to ensure Assumption (\ref{AssOrderOfMoments}) by applying the above auxiliary lemma.

It remains to show (iii): Fix $ \gamma > 0 $ and $ G $ denote the distribution of the randomly chosen weights of the ELM and let $ Q_{\matZ_T} $ be the law of the input variables $ \matZ_T = ( \vecz_1, \ldots, \vecz_T ) $. Then, by independence of weights, inputs and errors and since $ \vecx_{tj}= g( b_j + \vecw_j^\top \vecz_t ) $, Markov's inequality and Theorem~\ref{ThCONS}, there exists a constant $ C >0 $ such that for any $ \delta > 0 $ 
\begin{align*}
	& \PP\left( \frac{1}{T} \left| \sum_{t=1}^T x_{tj} \epsilon_t^{(\nu)} \right| > \delta \right)
 	\le \delta^{-\gamma} \EE \left| \frac{1}{T}  \sum_{t=1}^T x_{tj} \epsilon_t^{(\nu)} \right|^\gamma \\
	& \quad = \delta^{-\gamma}  \int \int \EE_{\{ \bfeps^{(\nu)} \}_{t=1}^T} \left|  \frac{1}{T} \sum_{t=1}^T  g( b_j + \vecw_j^\top \vecz_t ) \epsilon_t^{(\nu)} \right|^\gamma  \, dG( \vecw_j, b_j ) dQ_{\matZ_T}( \vecz_1, \ldots, \vecz_T )\\
	& \quad \le \delta^{-\gamma}  \iint C T^{\gamma/2} \, dG dQ_{\matZ_T} \\
	& \quad = O( T^{\gamma/2} ),
\end{align*}
where $\EE_{\{ \bfeps^{(\nu)} \}_{t=1}^T}  $ indicates that the expectation is w.r.t. the error process. For a deep ELM, $ x_{tj} = g_r^{(\matW_j^{(r)}, b_{jk})} \circ  \vecg_{r-1}^{(\matW^{(r-1)}, \vecb_{r-1})} \circ \cdots \circ  \vecg_1^{(\matW^{(1)}, \vecb_1)}( \vecz_t  ) $ is random but bounded, say, by $1$ for each random draw of the random weights $ \matW_j^{(k)} $, $ b_{jk} $, $ k = 1, \ldots, r $, see (\ref{DeepELM}). Therefore, again conditioning on $ \matW_j^{(k)} $, $ b_{jk} $, $ k = 1, \ldots, r $, by independence, 
\begin{align*}
 & \EE \left| \frac{1}{T}  \sum_{t=1}^T x_{tj} \epsilon_t^{(\nu)} \right|^\gamma  \\
 & = \quad \iint  \EE_{\{ \bfeps^{(\nu)} \}_{t=1}^T} \left|  \frac{1}{T} \sum_{t=1}^T g_r^{(\matW_j^{(r)}, b_{jk})} \circ  \vecg_{r-1}^{(\matW^{(r-1)}, \vecb_{r-1})} \circ \cdots \circ  \vecg_1^{(\matW^{(1)}, \vecb_1)}( \vecz_t  ) \epsilon_t^{(\nu)} \right|^\gamma \, dP_{\otimes \matW}dQ_{\matZ_T}
\end{align*}
where $ dP_{\otimes\matW} = \otimes_{k=1}^r d G_{k} $ with $ d G_k $ the distribution of the random weights (connection weights and biases) of the $k$th hidden layer. Now the integrand can be estimated as for a single hidden layer net.

\subsection{Proof of Theorem~\ref{THCONS2}}

The proof relies on the following exponential inequality for algebraically decaying $ \alpha $-mixing sequences, see \cite[p.~34, Remark~7~c)]{Doukhan1994}. 

\begin{lemma} 
\label{ExpIneqMix1}
	Let $ \{ \xi_t \} $ be a sequence of mean zero random variables with $ |\xi_t| \le 1 $ and $ \alpha $-mixing coefficients $ \alpha_\xi(k) $, $k \in \N $, satisfying $ \alpha_\xi(k) = O( k^{-v}) $ for some real $ v > 0 $. Then for any $ 0 < a < 1 $ there exists some real $ b > 0 $, such that for large enough $n$
	\[
	  \PP\left( \left|  \sum_{t=1}^n \xi_t \right| > x \sqrt{n}  \right) 
	  \le 2 \exp \left( - \frac{ b x \log^{1-a}( n ) }{ \sqrt{n} }  \right) 
	\]
\end{lemma}

Recall the representation
\[
	\wh{\bfbeta}_{T\nu} - \bfbeta_\nu  = ( T^{-1} \matX_T^\top \matX_T )^{-1} \frac{1}{T} \sum_{t=1}^T \vecx_t \epsilon_t^{(\nu)}.
\]
In view of (\ref{ConvInverseMatrix2}), consistency now follows if we show that
\[
  \left\|\frac{1}{T} \sum_{t=1}^T \vecx_t \epsilon_t^{(\nu)} \right\|_2 \stackrel{\PP}{\to} 0,
\]
as $ T \to \infty $. The union bound and Lemma~\ref{ExpIneqMix1} yields for $ \delta > 0 $
\begin{align*}
 \PP\left( \left\| \frac{1}{T} \sum_{t=1}^T \vecx_t \epsilon_t^{(\nu)} \right\|_2 > \delta  \right) 
 &= \PP\left(  \sum_{j=1}^p \left| \frac{1}{T} \sum_{t=1}^T x_{tj} \epsilon_t^{(\nu)}  \right|^2 > \delta^2  \right) \\
 & \le \PP\left( p \max_j \left| \frac{1}{T} \sum_{t=1}^T x_{tj} \epsilon_t^{(\nu)}  \right|^2 > \delta^2 \right) \\
 & \le \sum_{j=1}^p \PP\left( \left| \sum_{t=1}^T x_{tj} \epsilon_t^{(\nu)} \right| > \frac{\delta T}{\sqrt{p}}  \right) \\
 & \le 2 \exp\left( \log(p) - \frac{b \delta \log^{1-a}( T ) }{ \sqrt{p} } \right)
\end{align*}
A sufficient condition ensuring that the latter expression has the order $ o(1) $ is given by $ \log(p) \sqrt{p} = o\left( \log^{1-a}( T )  \right) $,
which completes the proof.

\subsection{Proof of Theorem~\ref{THCONS3}}

To establish Theorem~\ref{THCONS3} observe that in view of  (\ref{MixingCoefficientsEstimate2}), which holds for linear processes with coefficients depending on $t$, the exponential decay of the coeffients implies geometrically decaying $ \alpha $-mixing coefficients of the non-stationary and bounded sequences $ \{ x_{tj} \epsilon_t^{(\nu)} : t \ge 1 \} $, uniformly in $ j = 1, \ldots, p $. Recall the following results due to \cite{MerlevedeEtAl2009}; see  \cite[Corollary~24]{MerlevedePeligrad2013} for a related maximal inequality.

\begin{lemma} 	
\label{ExpIneqMixing}
	Let $ \{ \xi_t : t \ge 1 \} $ be a sequence of mean zero random variables bounded by $M$ and with $ \alpha $-mixing coefficients $ \alpha_\xi(k) $, $k \ge 1 $, satisfying $ \alpha_\xi(k) = O( \exp( - \gamma k ) $ for some $ \gamma > 0 $. Then there exists
	a constant $ C $  depending on $ \gamma $, such that for all $n \ge 4 $ and all $ x \ge 0 $
	\[
	  \PP\left( \left|  \sum_{t=1}^n \xi_t \right| > x \right) \le 2 \exp\left(  - \frac{ C x^2 }{ n M^2 + M x (\log n)( \log \log n) } \right).
	\]
\end{lemma}

The union bound yields
\begin{align*}
 \PP\left( \left\| \frac{1}{T} \sum_{t=1}^T \vecx_t \epsilon_t^{(\nu)} \right\|_2 > \delta  \right)
 & \le \sum_{j=1}^p \PP\left( \left| \sum_{t=1}^T x_{tj} \epsilon_t^{(\nu)} \right| > \frac{\delta T}{\sqrt{p}}  \right) 
\end{align*}
and hence by Lemma~\ref{ExpIneqMixing}
\begin{align}
\PP\left( \left\| \frac{1}{T} \sum_{t=1}^T \vecx_t \epsilon_t^{(\nu)} \right\|_2 > \delta  \right)& \le 2  \exp\left(  - \frac{ C \delta^2 T/p }{ 1 + \frac{\delta}{p^{1/2}} (\log T)( \log \log T) } + \log(p) \right) \nonumber \\
\label{ExpIneqCross}
& \le 2 \exp\left(  - \frac{ C \delta^2 T/p }{ 1 + \frac{\delta}{p^{1/2}} \log^2 T} + \log(p) \right) \\ \nonumber
& = o(1),
\end{align}
as $ T \to \infty $, if $ p \log (p) / T  = o(1) $ and $ \log^2(T) / \sqrt{p} = o(1) $, since then, for any sequence $
 f_T \to f > 0 $ (take $ f_T = C \delta^2 / (1+\delta \log^2(T) / \sqrt{p} ) $),  $ f_T^{-1} \log(p)  p/T < c < 1 $ for large enough $T$, which implies  $  f_T^{-1} p/T < 1/ \log(p) \Leftrightarrow f_T T/p  > \log(p) $ and hence $ T/p = \log(p) [ (T/( p  \log(p) )]  \to \infty $. 
 
 \subsection{Proof of Theorem~\ref{THCONSPRED}}
 
 As a preparation, recall the exponential bound in (\ref{ExpIneqCross}). For $ \lambda \le 1 $ we have
 \begin{align*}
\PP\left( \left\| \frac{1}{T} \sum_{t=1}^T \vecx_t \epsilon_t^{(\nu)} \right\|_2 > \lambda  \right) & \le  2 \exp\left(  - \frac{ C \lambda^2 T/p }{ 1 + \frac{1}{p^{1/2}} \log^2 T} + \log(p) \right). 
 \end{align*}
	Therefore, for $ \alpha \in (0,1) $ the choice
	\[
	  \lambda = C^{-1/2} \sqrt{ \log\left( \frac{2p}{ \alpha } \right)  \left( \frac{p}{T} + \frac{ \sqrt{p} \log^2(T) }{ T} \right) }
	\]
	guarantees
	\[
	  \PP\left( \left\| \frac{1}{T} \sum_{t=1}^T \vecx_t \epsilon_t^{(\nu)} \right\|_2 \le  \lambda  \right) \ge 1-\alpha.
	\]

In what follows, we draw on some standard arguments to bound the prediction error, see, for example, \cite{BuehlmannGeer2011}. 
 Since $ \bhbeta_{T\nu}^{(\ell_s)} $ minimizes $ \wt{\bfbeta} \mapsto T^{-1} \| \vecY_T^{(\nu)} - \matX_T \wt{\bfbeta} \|_2^2 + \pen_\lambda ( \wt{\bfbeta} ) $, we have
\[
   T^{-1} \| \vecY_T^{(\nu)} - \matX_T \bhbeta_{T\nu}^{(\ell_s)} \|_2^2  + \pen_\lambda(  \bhbeta_{T\nu}^{(\ell_s)} ) \le T^{-1} \| \vecY_T^{(\nu)} - \matX_T \bfbeta_\nu \|_2^2  + \pen_\lambda( \bfbeta_\nu ).
\]
Using $ \vecY_T^{(\nu)} - \matX_T \bhbeta_{T\nu}^{(\ell_s)}  = -\matX_T( \bhbeta_{T\nu}^{(\ell_s)} - \bfbeta_\nu ) + \bfeps^{(\nu)} $, the latter inequality is equivalent to
\begin{align*}
&  T^{-1}\bigl\{ 
 	\bfeps^{(\nu)}{}^\top \bfeps^{(\nu)} - 2 \bfeps^{(\nu)}{}^\top [ \matX_T( \bhbeta_{T\nu}^{(\ell_s)} - \bfbeta_\nu ) ] + (\bfbeta_\nu - \bhbeta_{T\nu}^{(\ell_s)} )^\top \matX_T^\top \matX_T (\bfbeta_\nu - \bhbeta_{T\nu}^{(\ell_s)} ) \\
& 	+ \pen_\lambda( \bhbeta_{T\nu}^{(\ell_s)} ) 
  \bigr\} \le \bfeps^{(\nu)}{}^\top \bfeps^{(\nu)}  + \pen_\lambda( \bfbeta_\nu  ).
\end{align*}
Collecting terms we obtain
\[
  - (2/T) \bfeps^{(\nu)}{}^\top [ \matX_T( \bhbeta_{T\nu}^{(\ell_s)} - \bfbeta_\nu ) ] + T^{-1} \| \matX_T( \bhbeta_{T\nu}^{(\ell_s)} - \bfbeta_\nu  ) \|_2^2 + \lambda \| \bhbeta_{T\nu}^{(\ell_s)} \|_2
  \le \pen_\lambda(  \bfbeta_\nu ) - \pen_\lambda(  \bhbeta_{T\nu}^{(\ell_s)}  ),
\]
and eventually arrive at
\[
  T^{-1} \| \matX_T( \bhbeta_{T\nu}^{(\ell_s)} - \bfbeta_\nu ) \|_2^2 \le (2/T) [\matX_T \bfeps^{(\nu)} ]^\top ( \bhbeta_{T\nu}^{(\ell_s)} - \bfbeta_\nu ) + \pen_\lambda( \bfbeta_\nu ) - \pen_\lambda(  \bhbeta_{T\nu}^{(\ell_s)}  ).
\]
We may now conclude that with probability at least $1-\alpha $
\begin{align*}
  T^{-1} \| \matX_T( \bhbeta_{T\nu}^{(\ell_s)} - \bfbeta_\nu ) \|_2^2
  & \le  \lambda \| \bhbeta_{T\nu}^{(\ell_s)} - \bfbeta_\nu \|_2  + \pen_\lambda( \bfbeta_\nu ) - \pen_\lambda(  \bhbeta_{T\nu}^{(\ell_s)}  ) 
\end{align*}
Recalling that $ \pen_\lambda( \cdot  ) = \lambda \| \cdot \|_s $ with $ 1 \le s\le 2 $,
the triangle inequality and the fact that $ \| \vecx \|_2 \le \| \vecx \|_s $ for $ 1 \le s \le 2 $ now provide the bound
\[
  T^{-1} \| \matX_T( \bhbeta_{T\nu}^{(\ell_s)} - \bfbeta_\nu ) \|_2^2 \le 2\lambda \| \bfbeta_\nu \|_s. 
\]
Lastly, consistency follows when $ \| \bfbeta_\nu \|_s $ is of smaller order than $ \lambda $.

 \subsection{Proof of Theorem~\ref{THCONSPRED2}}
 
 The result follows along the lines of the proof of Theorem~\ref{THCONSPRED} using the bound (\ref{CrossMeanLittleO1}) with $ \delta = \lambda$, equating the right-hand side with $ 1-\alpha $ and solving for $ \lambda $. The rest of the proof is unchanged.

\subsection{Proof of Theorem~\ref{THRIDGE2}}

The result can be shown arguing as in \cite{LiuYu2013}. Consider the SVD of
\[
\wt{\matX}_T = \frac{1}{\sqrt{T}} \matX_T = \matU \matD \matV^\top
\]
where $ \matU $ is a $ T \times T $ orthogonal matrix, $ \matV $ a $ p \times p $ orthogonal matrix and
\[
\matD = \diag( \rho_1, \ldots, \rho_p )
\]
the diagonal matrix of the eigenvalues of $ \wt{\matX}_T^\top \wt{\matX}_T = T^{-1} \matX_T^\top \matX_T $. The ridge estimator attains the representation
\[
\bhbeta_{T\nu}^{(R)} = \frac{1}{\sqrt{T}} \left( \wt{\matX}_T^\top \wt{\matX}_T + \frac{\lambda_T}{T} \matid  \right)^{-1} \wt{\matX}_T^\top \vecY^{(\nu)}
\]
Observing that $ \wt{\matX}_T^\top \wt{\matX}_T + \frac{\lambda_T}{T} \matid = \matV (\matD^2 + \frac{\lambda_T}{T} \matid ) \matV^\top $ with inverse
$ \matV \diag\left( \frac{1}{\rho_1^2 + \lambda_T/T}, \ldots, \frac{1}{\rho_p^2+ \lambda_T/T} \right) \matV^\top $, some algebra shows that 
\[
\bhbeta_{T\nu}^{(R)} = \matV \diag\left( \frac{\rho_1}{\rho_1^2+\lambda_T/T}, \ldots, \frac{\rho_p}{\rho_p^2 + \lambda_T/t} \right) \matU^\top \vecY^{(\nu)}.
\]
Since, similarly,
\[
	\bhbeta_{T\nu} = \matV \diag\left( 1 / \rho_1,  \ldots,  1/ \rho_T \right) \matU^\top \vecY^{(\nu)},
\]
we obtain 
\[
\sqrt{T}( \bhbeta_{T\nu}^{(R)} - \bhbeta_{T\nu}) = \matV 
\diag\left( - \frac{T \lambda_T}{T \rho_1 ( T\rho_1 + \lambda_T)}, \ldots, - \frac{T 	\lambda_T}{T \rho_p ( T\rho_1 + \lambda_T)} \right) 
\matU^\top \vecY^{(\nu)},
\]
such that
\begin{align*}
\| \sqrt{T}( \bhbeta_{T\nu}^{(R)} - \bhbeta_{T\nu})  \|_2^2
& \le \left\| \diag\left( - \frac{T \lambda_T}{T \rho_1 ( T\rho_1 + \lambda_T)}, \ldots, - \frac{T \lambda_T}{T \rho_p ( T\rho_1 + \lambda_T)} \right) \right\|_2^2 \\
& \le \max_{1 \le j \le p} \frac{ \lambda_T^2 }{ T^2 \rho_j^4 } \| \vecY^{(\nu)} \|_2^2 \\
& \le \frac{ \lambda_T^2 }{ T^2 \lambda_{\min}^4( T^{-1} \matX_T^\top \matX_T ) } \| \vecY^{(\nu)} \|_2^2 .
\end{align*}
Since $ \| \matX_T \bfbeta_\nu \| = O(T) $ and $ \EE \| \bfeps^{(\nu)} \|_2^2 = \sum_{t=1}^T \EE( \epsilon_t^{(\nu)} )^2 = O(T) $, if $ T, p  \to \infty $, in view of the assumption that $ \Var( \epsilon_t^{(\nu)}  ) = O(1) $, we may conclude, for any $ \delta > 0 $, by Markov's inequality
\begin{align*}
\PP\left( \| \sqrt{T}( \bhbeta_T^{(R)} - \bhbeta_{T\nu})  \|_2 > \delta  \right) 
&\le \frac{ \lambda_T^2 }{ T^2 \lambda_{\min}^4( T^{-1} \matX_T^\top \matX_T ) } \frac{ \EE( \| \vecY^{(\nu)} \|_2^2 ) }{\delta^2} \\
& = O\left( \frac{ \lambda_T^2 }{ T^2 } \EE( \| \vecY^{(\nu)} \|_2^2  \right) \\
& = O \left( \lambda_T^2 \right).
\end{align*}
For random regularization parameters $ \lambda_T $ the bound
\[
\| \sqrt{T}( \bhbeta_{T\nu}^{(R)} - \bhbeta_{T\nu})  \|_2 \le 
\frac{ \lambda_T }{ T \lambda_{\min}^2( T^{-1} \matX_T^\top \matX_T ) } \| \vecY^{(\nu)} \|_2
\]
and Markov's inequality yield
\begin{align*}
\PP\left( \| \sqrt{T}( \bhbeta_T^{(R)} - \bhbeta_{T\nu})  \|_2 > \delta  \right) 
&\le \PP\left( \frac{ \lambda_T \| \vecY^{(\nu)} \|_2 }{ \lambda_{\min}^2( T^{-1} \matX_T^\top \matX_T ) }  > T\delta  \right) \\
&=O\left( \frac{\sqrt{\EE(\lambda_T^2)} \sqrt{\EE( \| \vecY^{(\nu)} \|_2^2 ) } }{ T \delta }  \right) \\
& = O\left( \frac{ \sqrt{\EE(\lambda_T^2)} }{ \delta } \right),
\end{align*}
which is $ o(1) $, if  $ \EE(\lambda_T^2) = o(1) $, $ T, p \to \infty $.

\subsection{Proof of Theorem~\ref{THRIDGE}}


We have the representation
\begin{align*}
	\wh{\bfbeta}_{T\nu}^{(R)} &= ( \matX_T^\top \matX_T + \lambda_T \matid )^{-1} \matX_T^\top \vecY_T^{(\nu)} \\
	& = \left( T^{-1} \matX_T^\top \matX_T + \frac{\lambda_T}{T} \matid \right)^{-1} T^{-1} \matX_T^\top \matX_T \bfbeta_\nu + \left( T^{-1} \matX_T^\top \matX_T + \frac{\lambda_T}{T} \matid \right)^{-1}  T^{-1} \matX_T^\top \bfeps^{(\nu)}.
\end{align*}
First, observe that since $ | \lambda_T/T - \lambda^0 | = o_\PP( 1/\sqrt{p} ) $, by Assumption (\ref{AssRegr1}) and since $ \| \matid \|_F = \sqrt{p} $,
\[
  \left\| T^{-1} \matX_T^\top \matX_T +  \frac{\lambda_T}{T} \matid - \left( \bfSigma_\vecx + \lambda^0 \matid  \right)  \right\|_2  \le 
  C \frac{p^\eta}{\sqrt{T}} + \left| \frac{\lambda_T}{T} - \lambda^0 \right|  \sqrt{p} = o_\PP(1),
\]
and by Lemma~\ref{ConvRateInverse} 
\[
\left\| \left( T^{-1} \matX_T^\top \matX_T +  \frac{\lambda_T}{T} \matid \right)^{-1} - \left( \bfSigma_\vecx + \lambda^0 \matid  \right)^{-1}  \right\|_2  \le 
C \frac{p^\eta}{\sqrt{T}} + \left| \frac{\lambda_T}{T} - \lambda^0 \right|  \sqrt{p} = o_\PP(1),
\]
as well. To treat the first term in the above representation of $ \wh{\bfbeta}_{T\nu}^{(R)} $ apply the inequality
\[
  \| \matA_T \matB_T - \matA \matB \|_2 \le \| \matA_T - \matA \|_2 ( \| \matB_T - \matB \|_2 + \| \matB \|_2 ) + \| \matA \|_2 \| \matB_T - \matB \|_2
\]
with $ \matA_T = \left( T^{-1} \matX_T^\top \matX_T + \frac{\lambda_T}{T} \matid \right)^{-1} $, $ \matA = (\bfSigma_\vecx^{-1} + \lambda^0 \matid )^{-1} $, $ \matB_T = T^{-1} \matX_T^\top \matX_T \bfbeta_\nu $ and $ \matB = \bfSigma_\vecx \bfbeta_\nu $.
Observe that the spectrum  of $ \bfSigma_\vecx(\lambda^0) = \bfSigma_\vecx + \lambda^0 \matid $ is the spectrum of $ \bfSigma_\vecx $ shifted by $ \lambda^0 $, such that $ \| \bfSigma_\vecx(\lambda^0) \|_2 = O(1) $ and $ \| \matA \| = \| \bfSigma_\vecx(\lambda^0)^{-1} \|_2 = O(1) $. Since $ \bfbeta_\nu $ ensures by assumption that $ \| \matB \|_2 =  \| \bfSigma_\vecx \bfbeta_\nu \|_2 = O(1) $, we may conclude that
under the assumptions of assertion (ii)
\begin{equation}
\label{ConvFirst}
\left\| \left( T^{-1} \matX_T^\top \matX_T + \frac{\lambda_T}{T} \matid \right)^{-1} T^{-1} \matX_T^\top \matX_T \bfbeta_\nu  - (\bfSigma_\vecx + \lambda^0 \matid )^{-1} \bfSigma_\vecx \bfbeta_\nu  \right\|_2 = o_\PP(1)
\end{equation}
and
\begin{equation}
\label{ConvSecond}
\left\| \left( T^{-1} \matX_T^\top \matX_T + \frac{\lambda_T}{T} \matid \right)^{-1}  -   (\bfSigma_\vecx^{-1} + \lambda^0 \matid )^{-1} \right\|_2 = o_\PP(1).
\end{equation}
Consequently,
\[
\wh{\bfbeta}_{T\nu} = (\bfSigma_\vecx + \lambda^0 \matid )^{-1} \bfSigma_\vecx \bfbeta_\nu + o_\PP(1),
\]
provided $ \left\| T^{-1} \sum_{t=1}^T \vecx_t \epsilon_t^{(\nu)} \right\|_2 = o_\PP(1) $, which holds under the stated assumptions, as shown in (\ref{CrossMeanLittleO1}).  This verifies assertion (ii). To show (i) observe the representation
\begin{align*}
	\wh{\bfbeta}_{T\nu}^{(R)} - \bfbeta_{\nu} 
	& = \left[ \left( T^{-1} \matX_T^\top \matX_T + \frac{\lambda_T}{T} \matid \right)^{-1} T^{-1} \matX_T^\top \matX_T - \matid   \right] \bfbeta_\nu  \\
	& \qquad + \left( T^{-1} \matX_T^\top \matX_T + \frac{\lambda_T}{T} \matid \right)^{-1} \frac{1}{T} \sum_{t=1}^T \vecx_t \epsilon_t^{(\nu)}.
\end{align*}
In view of  $ \left\| \frac{1}{T} \sum_{t=1}^T \vecx_t \epsilon_t^{(\nu)} \right\|_2 = o_\PP(1) $, assertion (i) follows, if we show 
\begin{equation}
	\label{Consis1}
	\left\|  \left[ \left( T^{-1} \matX_T^\top \matX_T + \frac{\lambda_T}{T} \matid \right)^{-1} T^{-1} \matX_T^\top \matX_T \bfSigma_\vecx^{-1} - \bfSigma_\vecx^{-1} \right] \bfSigma_\vecx \bfbeta_\nu \right\|_2 = o(1)
\end{equation}
and
\begin{equation}
	\label{Consis2}
	\left\|  T^{-1} \matX_T^\top \matX_T + \frac{\lambda_T}{T} \matid   \right\|_{2} = O(1).
\end{equation}
The latter follows directly from (\ref{AssRegr1}), the estimate
\[
  \| T^{-1} \matX_T^\top \matX_T \|_{2} \le \| T^{-1} \matX_T^\top \matX_T - \bfSigma_\vecx \|_2 + \| \bfSigma_\vecx \|_{2} = O(1)
\] 
and 
\[
\left\| \frac{\lambda_T}{T} \matid  \right\|_F = \frac{\lambda_T \sqrt{p}}{T} = o_\PP(1).
\]
Also observe that
\[
\left\| T^{-1} \matX_T^\top \matX_T + \frac{\lambda_T}{T} \matid  - \bfSigma_\vecx   \right\|_{2} \le \left\|  T^{-1} \matX_T^\top \matX_T - \bfSigma_\vecx \right\|_2 + \frac{\lambda_T \sqrt{p}}{T} \le C \frac{p^\eta}{\sqrt{T}} + o_{\PP}(1),
\]
i.e., $ T^{-1} \matX_T^\top \matX_T + \frac{\lambda_T}{T} \matid   \to \bfSigma_\vecx $ in probability in the operator norm. Again applying Lemma~\ref{ConvRateInverse} we get
\[
\left\| \left( T^{-1} \matX_T^\top \matX_T + \frac{\lambda_T}{T} \matid \right)^{-1}  - \bfSigma_\vecx^{-1}  \right\|_{2} \le C \frac{p^\eta}{\sqrt{T}} + o_{\PP}(1).
\]
Now consider (\ref{Consis1}). Write
\begin{small}
\begin{align*}
	\matR_T & = \left[ \left( T^{-1} \matX_T^\top \matX_T + \frac{\lambda_T}{T} \matid \right)^{-1} T^{-1} \matX_T^\top \matX_T \bfSigma_\vecx^{-1} - \bfSigma_\vecx^{-1} \right] \bfSigma_\vecx \bfbeta_\nu \\
	&   = \left\{ \left[ \left( T^{-1} \matX_T^\top \matX_T + \frac{\lambda_T}{T} \matid   \right)^{-1} - \bfSigma_\vecx^{-1}  \right] T^{-1} \matX_T^\top \matX_T \bfSigma_\vecx^{-1}
	+ \bfSigma_\vecx^{-1} \left( T^{-1} \matX_T^\top \matX_T \bfSigma_\vecx^{-1} - \bfSigma_\vecx \bfSigma_\vecx^{-1}\right) \right\} \bfSigma_\vecx \bfbeta_\nu.
\end{align*}
\end{small}
Then, 
\begin{align*}
	\| \matR_T \|_2  &\le \left\| \left( T^{-1} \matX_T^\top \matX_T + \frac{\lambda_T}{T} \matid   \right)^{-1} - \bfSigma_\vecx^{-1} \right\|_2
	 \left\| T^{-1} \matX_T^\top \matX_T \bfSigma_\vecx^{-1} \right\|_{2}  \| \bfSigma_\vecx \bfbeta_\nu \|_{2} \\
	& + \| \bfSigma_\vecx^{-1} \|_2 \|  T^{-1} \matX_T^\top \matX_T - \bfSigma_\vecx  \|_2 \| \bfSigma_\vecx^{-1}  \|_{2} \| \bfSigma_\vecx \bfbeta_\nu \|_2.
\end{align*}
In view of (\ref{AssRegr1}) $  \left\| T^{-1} \matX_T^\top \matX_T  \right\|_F = O(1) $, and $ \| \bfSigma_\vecx^{-1} \|_{2} = 1/\lambda_{\min}( \bfSigma_\vecx ) = O(1) $ as well.  Therefore, 
$ \left\| T^{-1} \matX_T^\top \matX_T \bfSigma_\vecx^{-1} \right\|_{2}  = O(1) $.
Combining these facts with the assumption $ \| \bfSigma_\vecx \bfbeta_\nu \|_2 = O(1) $, we obtain $ \| \matR_T \|_2 = o_\PP(1) $, as $ T, p \to \infty $, which completes the proof.

\subsection{Proof of Theorem~\ref{ThASNOR} and Theorem~\ref{THASNORMRIDGE}}

We shall first consider the case $ m = 1 $ and $ \vecH = \id $ of a linear factor model for the errors. By using an interlacing embedding technique as  in \cite{Steland2020}, we will show that $ \bfeps_t = \vecZ_t $ is a strongly mixing linear process with coefficients decaying as $ O( j^{-7/2-\theta} ) $. This implies that the nonlinear process $ \vecH( \vecZ_t, \ldots, \vecZ_{t-m} ) $ is $ \alpha $-mixing as well with the same decay of the mixing coefficients and finite absolute moments of order $ 2+\delta $. Consequently, central limit theorems as well as exponential inequalities (as given above) can be applied for $ \vecH = \id $ as well as $ \vecH \not= \id $ under Assumption (\ref{AssGLIP}).  

By interlacing, we may write the $ \nu $th coordinate $ \epsilon_t^{(\nu)}  $  of $ \bfeps_t $ as a linear process with respect to the independent mean zero sequence  $ \ldots, \bfvareps_{t-1}^{(0)}{}^\top, \ldots,  \bfvareps_{t-1}^{(L)}{}^\top, \bfvareps_t^{(0)}{}^\top, \ldots, \bfvareps_t^{(L)}{}^\top $, denoted by $ Z_{tL'-k} $, $k \ge 0 $, where $ L' = (L+1)d $ is the length of the blocks of the interlaced $ L+1$ processes. At this point, recall the discussion following Assumption (\ref{AssDecay}). 
To simplify notation, absorb the distances into the coefficients and thus put $ \wt{\vecc}_j^{(\ell,\nu)} = \phi_{\nu \ell} d( \vecs_t, \vecsf_\ell) \vecc_j^{(\ell,\nu)} $ if $ \ell = 1, \ldots, L $ and $ \wt{\vecc}_j^{(0,\nu)} = \vecc_j^{(0,\nu)} $. Recalling that $ \epsilon_t^{(\nu)} = \sum_{j=0}^\infty \sum_{\ell=0}^L \wt{\vecc}_j^{(\ell,\nu)}{}^\top \bfvareps_{t-j}^{(\ell)} $, the interlacing is now implemented by stacking the innovations of the $L+1 $ processes,
\[
  (\ldots, \bfvareps_{t-1}^{(0)}{}^\top, \ldots,  \bfvareps_{t-1}^{(L)}{}^\top, \bfvareps_t^{(0)}{}^\top, \ldots, \bfvareps_t^{(L)}{}^\top ) =: ( \ldots,  Z_{(t-1)L'-L'-1}, \ldots Z_{(t-1)L'}, Z_{tL'-L'-1}, \ldots Z_{tL'} ),
\]
and matching the associated interlaced coefficients $ \ldots, \wt{\vecc}_{1}^{(0,\nu)}{}^\top, \ldots, \wt{\vecc}_{1}^{(L,\nu)}{}^\top, \wt{\vecc}_{0}^{(0,\nu)}{}^\top, \ldots, \wt{\vecc}_{0}^{(L,\nu)}{}^\top $ denoted by $ d_{tL',k} $, $ k \ge 0 $, i.e.
\[
  (\ldots,  \wt{\vecc}_{1}^{(0,\nu)}{}^\top, \ldots, \wt{\vecc}_{1}^{(L,\nu)}{}^\top, \wt{\vecc}_{0}^{(0,\nu)}{}^\top, \ldots, \wt{\vecc}_{0}^{(L,\nu)}{}^\top  )
  =: ( \ldots, d_{tL',L+L'}, \ldots, d_{tL',L+1}, d_{tL',L}, \ldots, d_{tL',0}  )
\]
Then we have the representations 
\[
  \epsilon_t^{(\nu)} = \sum_{j=0}^\infty \sum_{\ell=0}^L \wt{\vecc}_{j}^{(\ell,\nu)}{}^\top \bfvareps_{t-j}^{(\ell)} = \wt{\bfeps}_{tL'}, \qquad  \text{with} \quad \wt{\bfeps}_{s} = \sum_{k=0}^\infty d_{s,k} Z_{s-k}, 
\]
and
\[
  (\bfeps_1^{(\nu)}, \ldots, \bfeps_k^{(\nu)} ) = ( \wt{\bfeps}_{L'}, \ldots, \wt{\bfeps}_{k L'} )
\]
for all $ k \ge 1 $. Since $ L $ is finite, the decay condition on the $ c_{tj}^{(\ell)} $'s carries over to the $ d_{tk} $, such that $ d_{tk} = O( k^{-7/2-\theta} ) $.

Let $ \{ \xi_{\bm i} : \bm i \in \Z^q \} $ be a random field. Let $ I, J \subset \Z^q $ and define the associated mixing coefficient
\[
	\alpha_\xi(I, J) = \sup \{  | \PP(A \cap B ) - \PP(A) \PP(B) | : A \in \sigma( \xi_{\bm i} : {\bm i} \in I ), B \in \sigma( \xi_{\bm j} : {\bm j} \in J ) \},
\]
Further, for $ a, b \in \N \cup \{ \infty \} $ let 
\[
  \alpha_{a,b}( k ) = \{ \alpha_\xi(I, J) : | I | \le a, | J | \le b, d(I,J) \ge k \}
\]
where $ d(I,J) $ denotes the distance of the sets $I, J $. We use the following CLT for non-stationary random fields and time series, see  \cite{Bolthausen1982} and \cite[Th.~3.3.1]{Guyon1995}. It is worth mentioning the recent CLT of \cite{BradleyTone2017} under Lindeberg's condition and weak assumptions on the strong mixing coefficient, but the latter coefficient is with respect to half spaces and not studied for general linear random fields. 

\begin{theorem}
A possibly non-stationary array $ \xi_{\bm i} $, $ \bm i \ge \bm 1 $, with mixing coefficients $ \alpha_{a,b}(m) $ satisfies 
\[
\frac{S_{\bm n}}{\sqrt{Var(S_{\bm n})}} \stackrel{d}{\to} N(0,1),
\]
as $ \bm n \to \infty $, where $ S_{\bm n} = \sum_{\bm 1 \le \bm i \le \bm n} \xi_{\bm i} $, $ \bm n \ge \bm 1 $, and
\[ 
	\limsup_{\bm n \to \infty} \Var( {|\bm n|}^{-1/2} S_{\bm n} ) =  \limsup_{\bm n \to \infty} \frac{1}{|\bm n|} \sum_{\bm 1 \le \bm i, \bm j \le \bm n} | \Cov( \xi_{\bm i}, \xi_{\bm j} ) | < \infty,
\] 
if 
\begin{itemize}
	\item[(CLT-i)] $ \sup_{\bm i \ge 1} \EE | \xi_{\bm i} |^{2+\delta} < \infty $,
	\item[(CLT-ii)] $ \sum_{m=1}^\infty \alpha_{a,b}(m) < \infty $ for $ a+b \le 4 $, 
	$ \alpha_{1,\infty}(m) = o(m^{-1}) $ and for some $ \delta > 0 $
	\[
	\sum_{m=1}^{\infty} \left( \alpha_{1,1}(m) \right)^{\delta/(2+\delta)} < \infty,
	\]
	(e.g. $ \alpha_{a,b}(k) = O( k^{-2-\theta} ) $ for some $ \theta > 0$ and $ \delta = 2 $),
	\item[(CLT-iii)] $ \liminf_{\bm n \to \infty} \Var( {|\bm n|}^{-1/2} S_{\bm n} ) > 0 $.
\end{itemize}
\end{theorem}

For dimension $q = 1 $ of the index  domain one can replace the mixing coefficient $ \alpha_{1,\infty}(m)  $ in (CLT-ii),  since it is only used to estimate the $ L_1$-norm of the term  $ \sum_{j=1}^n \xi_{j}  e^{\iota \lambda ( \bar{S}_{n} - \bar{S}_{j,n})} $, where $ \bar{S}_{n} - \bar{S}_{j,n} = (\sigma_{n}
+o(1))^{-1/2} \sum_{i: |j-i| > m_n} \xi_{i} $, using the estimate $ \EE | \xi_{j}  e^{\iota \lambda ( \bar{S}_n - \bar{S}_{j,n})} | \le n \alpha_{1,\infty}(m_n) $, see the proof of \cite[Theorem~3.3.1]{Guyon1995}. But the latter bound can be replaced by 
\[  \EE | \xi_{j}  e^{\iota \lambda ( \bar{S}_n - \bar{S}_{j,n})} |  \le n \alpha^*(m_n ), \] 
where for $ m \in \N $,
\begin{align*} 
	\alpha^*(m) & := \sup_{i} \alpha( \sigma( \{ \xi_{i} \} ), \sigma( \{ \xi_{j} : |i-j| > m \} ) )  \\
		& \le  \sup_{i} \max \{ \alpha( \sigma( \xi_{j} : j \le i ), \sigma( \xi_{j} : j \ge i+m) ), \alpha( \sigma( \xi_{j} : j \ge i ), \sigma( \{ \xi_{j} : j \le i-m \} ) ) \} \\
		& \le \sup_{i} \max\{  \alpha( \calF_{-\infty}^i, \calF_{i+m}^\infty ), \alpha( \calF_{i}^\infty, \calF_{-\infty}^{i-m} ) \} \\
		& = \alpha(m),
\end{align*}
as already been noted in \cite{Bolthausen1982} (without derivation).

Fix $ \nu $ and let us check the above conditions for $ \xi_{t} = \epsilon_t^{(\nu)} $.
By Assumption (\ref{AssErrors2}) (CLT-i) holds, and by (\ref{AssErrors1}), (\ref{AssErrors2}) and (\ref{AssErrors5}) we can bound the $\alpha $-mixing coefficients by
\begin{equation}
\label{MixingCoefficientsEstimate}
\alpha_{a,b}( 2m ) = O\left( b \left[   \sup_{s \ge 1} \sum_{j>m} d_{sj}^2 \right]^{1/3}  \right),
\end{equation}
and
\begin{equation}
\label{MixingCoefficientsEstimate2}
  \alpha(m) = O\left(  \sup_{s \ge 0}   \sum_{i=m}^\infty \left[ \sum_{k=i}^\infty d_{sk}^2 \right]^{1/3} \right),
\end{equation}
cf. \cite{Gorodetskii1977}, \cite[Corollary~1, p.78]{Doukhan1994} and \cite[Corollary~1.7.3]{Guyon1995}. Next observe that,  by assumption (\ref{AssDecay}), 
 $ \max_{1 \le \ell \le L} | c_{j}^{(\ell)} | = O( j^{-7/2-\theta} )$, such that the coefficients of the embedding series inherit this decay: $ \sup_{s \ge 1} |d_{sj}| = O (j^{-7/2-\theta}) $. One can now easily check that this yields
\[
  m \alpha(m) = O( m^{-2\theta/3} ) 
\]
and 
\[
  \alpha_{a,b}( 2m) = O( m^{-2-2\theta/3} ) 
\]
for all $ a, b \in \N $, thus verifying (CLT-ii). 

Therefore, we may conclude that 
\[
  \frac{1}{\sqrt{T} v_{T\nu}} \sum_{t=1}^T \epsilon_t^{(\nu)} \stackrel{d}{\to} N(0,1),
\]
as $ T \to \infty $, where $ v_{T\nu}^2 = \Var( T^{-1/2} \sum_{t=1}^T \epsilon_t^{(\nu)}  ) $, since, by independence and (\ref{AssErrors4}), for $ d = 1 $
\begin{align*}
  v_{T\nu}^2 & = \Var\left( \sum_{\ell=1}^L  \frac{1}{\sqrt{T}} \sum_{t=1}^T \phi_{\nu \ell} d( \vecs_t, \vecr_\ell ) \vecF_t^{(\ell,\nu)} + \frac{1}{\sqrt{T}} \sum_{t=1}^T \vecE_t^{(\nu)}  \right) \\
    & \ge  \Var\left(  \frac{1}{\sqrt{T}} \sum_{t=1}^T \vecE_t^{(\nu)}  \right) \\
    & \to \sum_{j=0}^\infty ( c_j^{(0,\nu)} )^2 \sigma_{t0}^2 \\
    & \ge c^{(0,\nu)}(1)^2 \underline{\sigma}_0^2 > 0.
\end{align*}
For $ d > 1 $ one uses (\ref{AssErrors4GeneralD}) and argues analogously.

To show assertion (ii) of the theorem,  fix $ \veca \not= \vecnull $ and consider the statistic
\[
  W_T = \frac{1}{\sqrt{T}} \sum_{t=1}^T \veca^\top \vecx_t \epsilon_t^{(\nu)}. 
\]
Clearly,  $ \veca^\top \vecx_t \epsilon_t^{(\nu)} $, $ t \ge 0 $, is a non-stationary linear combination of the underlying linear processes where all coefficients, $ c_{j}^{(\nu)} $, are multiplied by $ \veca^\top \vecx_t $. By boundedness of $ \vecx_t $, the associated coefficients $ d_{tL',k} $, decay at the same rate $ O(k^{-7/2- \theta}) $ and therefore the $ \alpha $-mixing coefficients, $\alpha_{a,b}'(m)$,  of $ \veca^\top \vecx_t \epsilon_t^{(\nu)} $, $ t \ge 0 $, decay as $ O( m^{-2-2\theta/3} ) $, since (\ref{MixingCoefficientsEstimate}) and (\ref{MixingCoefficientsEstimate2}) allow for linear random fields with coefficients depending on $t$. This shows (CLT-ii). Further, condition (CLT-i) holds, because
\[
	\EE | \veca^\top \vecx_t \epsilon_t^{(\nu)} |^{2+\delta} \le C_\vecx \| \veca \|_2 \EE | \epsilon_t^{(\nu)} |^{2+\delta},
\]	
where $ C_\vecx $ is a norm bound for the regressors.
Observe that $ \EE( \veca^\top \epsilon_t^{(\nu)} ) =  0 $ and
\begin{align*}
  \Cov( \veca^\top \vecx_s \bfeps_s^{(\nu)}, \veca^\top \vecx_t \epsilon_t^{(\nu)} ) 
    & = \EE( \veca^\top \vecx_s \bfeps_s^{(\nu)} \veca^\top \vecx_t \epsilon_t^{(\nu)}  ) \\
    & = \EE(  \veca^\top \vecx_s \bfeps_s^{(\nu)} \epsilon_t^{(\nu)} \vecx_t^\top \veca ) \\
    & = \veca^\top \left( \vecx_s \vecx_t^\top  \EE( \bfeps_s^{(\nu)} \epsilon_t^{(\nu)} ) \right) \veca.
\end{align*}
Therefore, $ \eta_T^2 = \Var( W_T )  $ is given by
\[
  \eta_T^2 = \eta_T^2( \veca )  = \veca^\top \bfGamma_T^{(\nu)} \veca,
\]
by definition of $ \bfGamma_T^{(\nu)}  $, and thus (\ref{AssGXe}) ensures that $ \liminf_{T\to \infty} \eta_T > 0 $, which verifies condition (CLT-iii). It follows that
\[
	\vecW_T = \frac{1}{\sqrt{T}} \sum_{t=1}^T \vecx_t \epsilon_t^{(\nu)} \sim
\calA \calN\left( \vecnull, \bfGamma_T^{(\nu)} \right).
\] 
We proceed by showing that $ \vecW_T = O_\PP(1) $. To see this, recall that there exists a constant $ c_p > 0 $ such that $ \| \cdot \|_2 \le c_p \| \cdot \|_\infty $. This gives
\[
\PP\left( \left\| \vecW_T  \right\|_2 > M  \right)
\le \PP\left(  \max_j \left| \frac{1}{\sqrt{T}} \sum_{t=1}^T x_{tj} \epsilon_t^{(\nu)}  \right| > \frac{M}{c_p}  \right) \le \sum_{j=1}^p \PP\left(  \left| \frac{1}{\sqrt{T}} \sum_{t=1}^T x_{tj} \epsilon_t^{(\nu)}  \right| > \frac{M}{c_p}  \right).
\]
But, since $ \eta^2_{Tj} = \eta^2_T( \vece_j ) \le \lambda_{\max}( \bfGamma^{(\nu)}_T ) \le K $ for some constant $K$, where $ \vece_j $ denotes the $j$th unit vector,
\[
	\PP\left(  \left| \frac{1}{\sqrt{T}} \sum_{t=1}^T x_{tj} \epsilon_t^{(\nu)}  \right| > \frac{M}{c}  \right) \le
		\PP\left(  \left| \frac{1}{\eta_{Tj} \sqrt{T}} \sum_{t=1}^T x_{tj} \epsilon_t^{(\nu)}  \right| > \frac{M}{c K}  \right).
\]
The right-hand side can be made arbitrarily small for large $M$, because 
\[
	\frac{1}{\eta_{Tj} \sqrt{T}} \sum_{t=1}^T x_{tj} \epsilon_t^{(\nu)} \stackrel{d}{\to} N(0,1), 
\] 
as $ T \to \infty $. 
Now observe the representation
\[
\sqrt{T}( \wh{\bfbeta}_{T\nu} - \bfbeta_\nu ) = ( T^{-1} \matX_T^\top \matX_T )^{-1} \frac{1}{\sqrt{T}} \sum_{t=1}^T \epsilon_t^{(\nu)} \vecx_t.
\]
Since matrix inversion is a continuous transformation, (\ref{AssRegr1}) implies
\[ 
\| ( T^{-1} \matX_T^\top \matX_T )^{-1} - \bfSigma_\vecx^{-1} \|_F = o(1), \] 
as $ T \to \infty$. Combining this with $ \| \vecW_T \|_2 = O_\PP(1) $  we obtain
\begin{align*}
 \left\|  ( T^{-1} \matX_T^\top \matX_T )^{-1} \vecW_T  - \bfSigma_\vecx^{-1} \vecW_T  \right\|_F & \le \| ( T^{-1} \matX_T^\top \matX_T )^{-1} - \bfSigma_\vecx^{-1} \|_F \| \vecW_T \|_2 \\ &= o_P(1),
\end{align*}
as $ T \to \infty $. Consequently, 
\begin{align*}
\sqrt{T}( \wh{\bfbeta}_{T\nu} - \bfbeta_\nu ) 
&= \bfSigma_\vecx^{-1} \frac{1}{\sqrt{T}} \sum_{t=1}^T \vecx_t \epsilon_t^{(\nu)} + o_\PP(1) \\
& = \frac{1}{\sqrt{T}} \sum_{t=1}^T \bfSigma_\vecx^{-1} \vecx_t \epsilon_t^{(\nu)} + o_\PP(1).
\end{align*}
Repeating the above arguments to verify (CLT-i)-(CLT-iii) with $ \vecx_t $ replaced by $ \bfSigma_\vecx^{-1} \vecx_t $ (and hence $ \veca^\top \vecx_t $ replaced by $ \vecb^\top \vecx_t $ with $ \vecb^\top = \veca^\top \bfSigma_\vecx^{-1} \not= \vecnull $ since $ \bfSigma_\vecx > 0 $), we obtain the asymptotic normality of
$ \sqrt{T}( \wh{\bfbeta}_{Tj} - \bfbeta_j )  $ with asymptotic covariance matrices
$
  \bfSigma_\vecx^{-1} \bfGamma_T^{(\nu)} \bfSigma_\vecx^{-1}
$.

The asymptotic normality for the ridge estimator follows with minor modifications: In view of the representation
\[
\sqrt{T}( \wh{\bfbeta}_{T\nu}^{(R)} - \bfbeta_\nu ) 
= ( \matX_T^\top \matX_T + \frac{\lambda_T}{T} \matid )^{-1} \frac{1}{\sqrt{T}} \sum_{t=1}^T \vecx_t \epsilon_t^{(\nu)},
\]
we have by continuity of matrix inversion,
\[
\sqrt{T}( \wh{\bfbeta}_{T\nu}^{(R)} - \bfbeta_\nu ) 
= \left[ (\bfSigma_\vecx + \lambda^0 \matid )^{-1} + o_{\| \cdot \|_F}(1) \right]  \frac{1}{\sqrt{T}} \sum_{t=1}^T \vecx_t \epsilon_t^{(\nu)}.
\]
Therefore, the results follows using the same arguments as above.

%
%


\bibliographystyle{plain}
\bibliography{lit}

\begin{thebibliography}{10}

\bibitem{Bolthausen1982}
E.~Bolthausen.
\newblock On the central limit theorem for stationary mixing random fields.
\newblock {\em Ann. Probab.}, 10(4):1047--1050, 1982.

\bibitem{BradleyTone2017}
Richard~C. Bradley and Cristina Tone.
\newblock A central limit theorem for non-stationary strongly mixing random
  fields.
\newblock {\em J. Theoret. Probab.}, 30(2):655--674, 2017.

\bibitem{BuehlmannGeer2011}
Peter B\"{u}hlmann and Sara van~de Geer.
\newblock {\em Statistics for high-dimensional data}.
\newblock Springer Series in Statistics. Springer, Heidelberg, 2011.
\newblock Methods, theory and applications.

\bibitem{ChenDonoho1994}
S.~Chen and D.~Donoho.
\newblock Basis pursuit.
\newblock {\em In 28th Asilomar Conf. Signals, Systems Computers}, 1994.

\bibitem{EMNIST2017}
Gregory Cohen, Saeed Afshar, Jonathan Tapson, and André van Schaik.
\newblock {EMNIST}: an extension of {MNIST} to handwritten letters, 2017.

\bibitem{Doukhan1994}
Paul Doukhan.
\newblock {\em Mixing}, volume~85 of {\em Lecture Notes in Statistics}.
\newblock Springer-Verlag, New York, 1994.
\newblock Properties and examples.

\bibitem{Dudek2019}
Grzegorz Dudek.
\newblock Generating random weights and biases in feedforward neural networks
  with random hidden nodes.
\newblock {\em Inform. Sci.}, 481:33--56, 2019.

\bibitem{JMLR:v15:delgado14a}
Manuel Fern{{\'a}}ndez-Delgado, Eva Cernadas, Sen{{\'e}}n Barro, and Dinani
  Amorim.
\newblock Do we need hundreds of classifiers to solve real world classification
  problems?
\newblock {\em Journal of Machine Learning Research}, 15(90):3133--3181, 2014.

\bibitem{Gorodetskii1977}
V.~V. Gorodecki\u{\i}.
\newblock The strong mixing property for linearly generated sequences.
\newblock {\em Teor. Verojatnost. i Primenen.}, 22(2):421--423, 1977.

\bibitem{Guyon1995}
Xavier Guyon.
\newblock {\em Random fields on a network}.
\newblock Probability and its Applications (New York). Springer-Verlag, New
  York, 1995.
\newblock Modeling, statistics, and applications, Translated from the 1992
  French original by Carenne Lude\~{n}a.

\bibitem{GyoerfiKohlerEtAl2002}
L\'{a}szl\'{o} Gy\"{o}rfi, Michael Kohler, Adam Krzy\.{z}ak, and Harro Walk.
\newblock {\em A distribution-free theory of nonparametric regression}.
\newblock Springer Series in Statistics. Springer-Verlag, New York, 2002.

\bibitem{Higdon2002}
Dave Higdon.
\newblock Space and space-time modeling using process convolutions.
\newblock In {\em Quantitative methods for current environmental issues}, pages
  37--56. Springer, London, 2002.

\bibitem{HuangSongGuptaWu2014}
G.~{Huang}, S.~{Song}, J.~N.~D. {Gupta}, and C.~{Wu}.
\newblock Semi-supervised and unsupervised extreme learning machines.
\newblock {\em IEEE Transactions on Cybernetics}, 44(12):2405--2417, 2014.

\bibitem{Huang2004}
G.-B. Huang, Q.-Y. Zhu, and C.-K. Sieq.
\newblock Extreme learning machine: a new learning scheme of feedforward neural
  networks.
\newblock {\em Neural Networks}, pages 985--990, 2004.

\bibitem{Huang2006}
G.-B. Huang, Q.-Y. Zhu, and C.-K. Siew.
\newblock Extreme learning machine: theory and applications.
\newblock {\em Neurocomputing}, 70(1):489--501, 2006.

\bibitem{HuangWang-2011}
Guang-Bin Huang, Dian~Hui Wang, and Yuan Lan.
\newblock Extreme learning machines: a survey.
\newblock {\em International Journal of Machine Learning and Cybernetics},
  2(2):107--122, 2011.

\bibitem{KnightFu2000}
Keith Knight and Wenjiang Fu.
\newblock Asymptotics for lasso-type estimators.
\newblock {\em Ann. Statist.}, 28(5):1356--1378, 2000.

\bibitem{LedoitWolf2004}
Olivier Ledoit and Michael Wolf.
\newblock A well-conditioned estimator for large-dimensional covariance
  matrices.
\newblock {\em J. Multivariate Anal.}, 88(2):365--411, 2004.

\bibitem{Silva2014}
Jo\~{a}o Lita~da Silva.
\newblock Some strong consistency results in stochastic regression.
\newblock {\em J. Multivariate Anal.}, 129:220--226, 2014.

\bibitem{LiuYu2013}
Hanzhong Liu and Bin Yu.
\newblock Asymptotic properties of {L}asso+m{LS} and {L}asso+{R}idge in sparse
  high-dimensional linear regression.
\newblock {\em Electron. J. Stat.}, 7:3124--3169, 2013.

\bibitem{LiuEtAl2015}
Xia Liu, Shaobo Lin, Jian Fang, and Zongben Xu.
\newblock Is extreme learning machine feasible? {A} theoretical assessment
  ({P}art {I}).
\newblock {\em IEEE Trans. Neural Netw. Learn. Syst.}, 26(1):7--20, 2015.

\bibitem{MeinshausenBuehlmann2006}
Nicolai Meinshausen and Peter B\"{u}hlmann.
\newblock High-dimensional graphs and variable selection with the lasso.
\newblock {\em Ann. Statist.}, 34(3):1436--1462, 2006.

\bibitem{MeinshausenYu2009}
Nicolai Meinshausen and Bin Yu.
\newblock Lasso-type recovery of sparse representations for high-dimensional
  data.
\newblock {\em Ann. Statist.}, 37(1):246--270, 2009.

\bibitem{MerlevedePeligrad2013}
Florence Merlev\`ede and Magda Peligrad.
\newblock Rosenthal-type inequalities for the maximum of partial sums of
  stationary processes and examples.
\newblock {\em Ann. Probab.}, 41(2):914--960, 2013.

\bibitem{MerlevedeEtAl2009}
Florence Merlev\`ede, Magda Peligrad, and Emmanuel Rio.
\newblock Bernstein inequality and moderate deviations under strong mixing
  conditions.
\newblock In {\em High dimensional probability {V}: the {L}uminy volume},
  volume~5 of {\em Inst. Math. Stat. (IMS) Collect.}, pages 273--292. Inst.
  Math. Statist., Beachwood, OH, 2009.

\bibitem{RahimiRecht2007}
Ali Rahimi and Benjamin Recht.
\newblock Random features for large-scale kernel machines.
\newblock In J.~Platt, D.~Koller, Y.~Singer, and S.~Roweis, editors, {\em
  Advances in Neural Information Processing Systems}, volume~20. Curran
  Associates, Inc., 2008.

\bibitem{RahimiRecht2008a}
Ali Rahimi and Benjamin Recht.
\newblock Uniform approximation of functions with random bases.
\newblock In {\em 2008 46th Annual Allerton Conference on Communication,
  Control, and Computing}, pages 555--561, 2008.

\bibitem{RahimiRecht2008b}
Ali Rahimi and Benjamin Recht.
\newblock Weighted sums of random kitchen sinks: Replacing minimization with
  randomization in learning.
\newblock In Daphne Koller, Dale Schuurmans, Yoshua Bengio, and Léon Bottou,
  editors, {\em NIPS}, pages 1313--1320. Curran Associates, Inc., 2008.

\bibitem{RenBanerjee2013}
Qian Ren and Sudipto Banerjee.
\newblock Hierarchical factor models for large spatially misaligned data: a
  low-rank predictive process approach.
\newblock {\em Biometrics}, 69(1):19--30, 2013.

\bibitem{SchmidtKraaijveldDuin1992}
W.~F. Schmidt, M.~A. Kraaijveld, and R.~P.~W. Duin.
\newblock Feedforward neural networks with random weights.
\newblock In {\em Proceedings of 11th IAPR Int. Conf. on Pattern Recognition},
  volume~2, page~1, 1992.

\bibitem{SchmidtEtAl1992}
W.F. Schmidt, M.A. Kraaijveld, and R.P.W Duin.
\newblock Feedforward neural networks with random weights.
\newblock {\em Proceedings of 11th IAPR}, 2:1--4, 1992.

\bibitem{Steland2020}
Ansgar Steland.
\newblock Testing and estimating change-points in the covariance matrix of a
  high-dimensional time series.
\newblock {\em J. Multivariate Anal.}, 177, 2020.

\bibitem{TangDengHuang2016}
J.~{Tang}, C.~{Deng}, and G.~{Huang}.
\newblock Extreme learning machine for multilayer perceptron.
\newblock {\em IEEE Transactions on Neural Networks and Learning Systems},
  27(4):809--821, 2016.

\bibitem{Tib1996}
Robert Tibshirani.
\newblock Regression shrinkage and selection via the lasso.
\newblock {\em J. Roy. Statist. Soc. Ser. B}, 58(1):267--288, 1996.

\bibitem{Vershynin2018}
Roman Vershynin.
\newblock {\em High-dimensional probability}, volume~47 of {\em Cambridge
  Series in Statistical and Probabilistic Mathematics}.
\newblock Cambridge University Press, Cambridge, 2018.
\newblock An introduction with applications in data science, With a foreword by
  Sara van de Geer.

\bibitem{Wainwright2019}
Martin~J. Wainwright.
\newblock {\em High-dimensional statistics}, volume~48 of {\em Cambridge Series
  in Statistical and Probabilistic Mathematics}.
\newblock Cambridge University Press, Cambridge, 2019.
\newblock A non-asymptotic viewpoint.

\bibitem{Weyl1912}
Hermann Weyl.
\newblock Das asymptotische {V}erteilungsgesetz der {E}igenwerte linearer
  partieller {D}ifferentialgleichungen (mit einer {A}nwendung auf die {T}heorie
  der {H}ohlraumstrahlung).
\newblock {\em Math. Ann.}, 71(4):441--479, 1912.

\bibitem{YuWangSamworth2015}
Y.~Yu, T.~Wang, and R.~J. Samworth.
\newblock A useful variant of the {D}avis-{K}ahan theorem for statisticians.
\newblock {\em Biometrika}, 102(2):315--323, 2015.

\end{thebibliography}

\end{document}